\documentclass[sigconf]{acmart}
\copyrightyear{2026}
\acmYear{2026}
\setcopyright{cc}
\setcctype{by}
\acmConference[KDD '26]{Proceedings of the 32nd ACM SIGKDD Conference on Knowledge Discovery and Data Mining V.2}{August 09--13, 2026}{Jeju Island, Republic of Korea}
\acmBooktitle{Proceedings of the 32nd ACM SIGKDD Conference on Knowledge Discovery and Data Mining V.2 (KDD '26), August 09--13, 2026, Jeju Island, Republic of Korea}
\acmDOI{10.1145/3770855.3818470}
\acmISBN{979-8-4007-2259-2/2026/08}

\AtBeginDocument{%
  }

\newcommand{\merit}{\mathrm{cont}}
\newcommand{\belief}{\mathrm{hist}}
\newcommand{\rc}{\mathrm{RC}}
\newcommand{\anc}{\mathrm{anc}}

\newtheorem{theorem}{Theorem}
\newtheorem{lemma}{Lemma}

\newtheorem{proposition}{Proposition}
\newtheorem{corollary}{Corollary}

\usepackage{algorithm}
\usepackage{algorithmic}

\begin{document}
\setlength{\abovedisplayskip}{3pt}
\setlength{\belowdisplayskip}{3pt}
\setlength{\abovedisplayshortskip}{2pt}
\setlength{\belowdisplayshortskip}{2pt}
\title{Representation Curriculum: Stagewise Training for Robust Ranking and Allocation}

\author{Ehsan Ebrahimzadeh}
\author{Sina Baharlouei}
\author{Abraham Bagherjeiran}
\affiliation{
  \institution{eBay Search Ranking and Monetization}
  \city{San Jose}
  \state{California}
  \country{USA}
}

\renewcommand{\shortauthors}{Ebrahimzadeh, Baharlouei, Bagherjeiran} 

\begin{abstract}
Ranking in digital marketplaces is a dynamic exposure-allocation mechanism: displayed items shape discovery trajectories and success events, which are logged by the platform to update future allocation policies.
Modern ranking systems therefore rely heavily on endogenous, exposure-confounded signals (e.g. popularity estimates, CTR/CVR aggregates, and ID-based representation), because they are highly predictive under approximately stationary demand and explain substantial variance in logged outcomes. 
Yet this predictive power can become a learning shortcut: early access to exposure-dependent \emph{belief} signals in training steers optimization dynamics toward over-reliance on them and away from exposure-independent \emph{merit} signals (e.g. content-based competitiveness and semantic intent affinity estimates). Consequently, the learned policy tends to entrench incumbents and degrade cold-start generalization and robustness under distribution shift.

We propose \textit{Representation Curriculum} (\textbf{RC}), a semantics-aware training-time optimization-trajectory intervention that temporally stages feature utilization.
RC foregrounds content-based merit signals in the initial stage of training, then introduces exposure-dependent historic belief signals while anchoring the content pathway to remain close to the learned merit representation, curbing shortcut reliance on historical signals and mitigating gradient starvation on content signals. 
We formalize RC independently of task and hypothesis class and provide ranking-specific instantiations. In a Gaussian linear ridge setting, we derive closed-form solutions and verifiable sufficient conditions under which RC strictly reduces population risk on a welfare-aligned cold-start target distribution, with a quantified Pareto tradeoff against source performance.
Experiments on public learning-to-rank and recommendation benchmarks, together with randomized online experiments in a large-scale e-commerce product search system, show that RC measurably shifts reliance from historical belief signals toward content-based merit signals and yields consistent gains on cold populations with a controlled trade-off in head performance.
\end{abstract}

\received[accepted]{1 June 2026}

\maketitle
\section{Introduction}
\label{sec:intro}
Modern ranking systems in search and recommendation make allocation decisions that shape how consumers discover inventory, how suppliers receive exposure, and how the marketplace evolves over time.
Because these policies mediate outcomes for multiple stakeholders, they are expected to account for the incentives of suppliers, consumers, and the platform and generalize reliably across heterogeneous traffic segments.
This is especially challenging in dynamic marketplaces, with high novelty, heterogeneity, and churn on the supply side, and non-stationary demand and heterogeneous intent on the consumer side.
It is therefore crucial to maintain tighter control on \emph{how} \emph{decision-making primitives} are used to make welfare aligned allocations.
When rankers rely primarily on exposure-shaped historical signals that are endogenous to the platform’s past allocations, they can exhibit self-reinforcing generalization patterns that disproportionately favor incumbents and weaken discovery, fairness, and long-run platform health.
Popularity-driven generalization can lead to: (i) cold-start fragility that harms discovery and new supplier on-boarding, (ii) shifting the burden of discovery to users rather than relying on exposure-independent content-based merit signals, (iii) contextual intent misalignment when feedback is pooled across contexts to obtain statistically stable estimates, and (iv) self-reinforcing exposure loops that entrench incumbents\cite{chaney2018algorithmic}.

Existing approaches mitigate this \emph{popularity bias}\cite{chen2023biasdebias,abdollahpouri2019managingpopularity} through: (i) correcting training data selection effects counterfactual learning techniques like propensity weighting; (ii)  shaping supervision through distillation or reward modeling to encourage desired allocation behavior; (iii) shrinkage and uncertainty-aware belief estimation (e.g., empirical Bayes); (iv) architectural constraints that limit history–content fusion with content pathways; and (v) stochastic masking/dropout as regularization techniques in training.
These methods are broadly complementary, but they are typically oblivious to data semantics and \emph{how} learning allocates capacity across feature groups over time.
We argue, however, that early access to belief signals can dominate gradients and induce a stable shortcut solution that under-learns content-based merit and propose a semantics-aware training trajectory intervention: delay access to belief signals and protect the content pathway, enabling explicit tradeoffs between head performance on logged data and robustness on welfare-aligned target segments.
\vspace{-2mm}
\subsection{Prior art}
\label{sec:intro-prior}


\emph{(1) Counterfactual and causal learning from endogenous logs.}
Counterfactual learning to rank(LTR) approaches account for biased feedback (e.g., position bias) in empirical risk minimization via propensity weighting \cite{swaminathan2015counterfactual,joachims2017unbiasedltr,schnabel2016recommendationsastreatments}.
Structural causal models treat popularity as a confounder and intervene explicitly \cite{zhang2021causal,wei2021model}.
The broader goal of robustness to distribution shift has been studied under invariant prediction\cite{peters2016invariant}, invariant risk minimization\cite{arjovsky2019irm}, and anchor regression \cite{rothenhausler2021anchor}.
These works target identification and evaluation, but typically treat feature representations as given and do not directly control how the learner utilizes different feature groups.

\emph{(2) Supervision-based behavior shaping.}
Platforms can encode preferences via distillation from a merit-centric teacher \cite{ebrahimzadeh2024kd,hinton2015distilling}, constrained optimization for exposure fairness \cite{singh2018exposure,singh2019policyfair,singh2021uncertaintyfair,tang2024multi}, or reward shaping and policy imitation\cite{chen2022off}.
These methods shape behavior through objectives, but do not necessarily prevent shortcut learning when belief signals dominate early optimization.

\emph{(3) History stabilization and shrinkage.}
Industrial systems use uncertainty-aware belief estimates (e.g., Empirical Bayes smoothing with global or content-based priors) to aggregate platform belief for LTR and product search \cite{yang2022clickslabels,yang2024ebrank,han2022ebayesproductsearch, ardywibowo2025bayescns}.
EB improves the quality of exposure-dependent signals via reduces variance and calibrated uncertainty,  but it does not directly address the learning dynamics problem and belief signals can still dominate learning.

\emph{(4) Architectural constraints and explicit decompositions.}
Multi-tower architectures, gating, and limited fusion constrain how historic belief and content-based merit representations combine\cite{volkovs2017dropoutnet,wang2018lrmm}.
They can enforce interpretable policy structure but may under-perform if explicit decomposition is too strict.

\emph{(5) Optimization-trajectory interventions.}
Training dynamics influence the learned representations and strong predictors can suppress gradient signal for weaker cues in neural networks \cite{pezeshki2021gradientstarvation}.
Generic regularizers (dropout) and stochastic masking reduce reliance on any single feature group \cite{srivastava2014dropout,volkovs2017dropoutnet,wang2018lrmm}, but are usually semantics-agnostic and not welfare aligned. Our approach differs from standard curriculum learning\cite{ferro2018continuation, zeng2022curriculum} in that we shape the properties of the model by explicit anchoring to the content pathway.
\vspace{-2mm}
\subsection{Contributions}
\label{sec:intro-contrib}

\paragraph{A semantics-aware representation curriculum for robust allocation:}
We propose \emph{Representation Curriculum} (RC), which partitions features into content-based merit signals and exposure-dependent historic belief signals and temporally stages the training.
Stage~1 learns a model trained only on content-based signals.
Stage~2 introduces historic signals but anchors the optimization so that content-based competence is preserved.
This feature-access curriculum complements the standard approaches to mitigate over-reliance on historic belief signals by controlling how the learner relies on them.

\paragraph{Theoretical Guarantees with quantified trade-offs and Verifiable Sufficient Conditions:}
We analyze RC in linear ridge regression, providing verifiable conditions under which RC reduces population risk on a target distribution that emphasizes cold/strategic segments, while remaining competitive on the logged data distribution.
We extend the lens to boosted trees via base-margin curricula and to neural ranking, where we connect pairwise/listwise saturation to feature starvation and provide practical diagnostics.

\paragraph{Experimental Evidence on Public Benchmarks and Deployed Systems:}
On MSLR-WEB, we identify highly predictive historic belief based features that encode behavioral evidence, and show how RC reduces feature importance on these historic belief signals and improves cold-item ranking quality with minimal impact on overall performance.
On MovieLens, RC shapes neural two-tower recommenders toward stronger content reliance and yields stronger frozen-start performance in the absence of historic signals.
Finally, an online A/B test in a major e-commerce sponsored search system establishes that a policy trained via RC increases exposure and sale velocity for new listings with neutral aggregate KPIs, situating RC as an effective behavior shaping technique in large scale systems.
\vspace{-1mm}
\section{Problem Setup: Robust Policy Updates under Endogenous History Signals}
\label{sec:setup}
We formalize ranking as an allocation mechanism with two semantically distinct classes of features:
\emph{exposure-independent, content-based merit signals} and \emph{exposure-dependent historical belief signals}.
At iteration $t$, a deployed policy $\pi_t$ maps a context $c$ (query, session, user intent state) and a candidate set to a ranked slate(or a distribution over rankings).
Exposure induces interactions, and a data sampling policy produces a logged dataset $\mathcal{D}_t$, which is, in turn, used to train a new policy $\pi_{t+1}$; that is
\begin{equation}
\pi_t \ \longrightarrow\  \mathcal{D}_t \sim \mathcal{P}(\pi_t)\ \longrightarrow\  \pi_{t+1},
\label{eq:setup-loop}
\end{equation}
where $\mathcal{P}(\pi_t)$ denotes the data distribution induced by deployment (including selection effects and feedback).
A key challenge is that $\mathcal P(\pi_t)$ is \emph{endogenous}: which contexts and items appear, and which outcomes are observed, are shaped by the deployed policy.
Platforms rarely optimize solely for the population represented in logged data. Instead, they choose a \emph{target distribution} $\mathcal Q$ that reflects strategic or welfare-aligned goals (e.g., improving long-tail coverage, supporting new supply, prioritizing high-value intents, or meeting fairness constraints).
Our aim is to train the next policy so that it generalizes to a platform-chosen target population and does not over-rely on endogenous belief features. If each update is robust to the $\mathcal{P}(\pi_t)$--$\mathcal{Q}$ mismatch, then repeated deployment is less likely to amplify feedback loops. Over time, this helps maintain alignment with long-run platform objectives.

Each training example is a tuple $\mathbf{z}=(c,i,y,\mathbf{x})$, where $c$ is a context, $i$ is an item, $y$ is a label (e.g., relevance, click, conversion), and $\mathbf{x}$ is a feature vector.
We assume the features decompose into two semantic groups, $\mathbf{x} \equiv (\mathbf{x}_C, \mathbf{x}_H)$, where
\begin{itemize}
  \item $\mathbf{x}_C$ are \textbf{content-based signals}, meaning features available regardless of prior exposure (e.g., query--item affinity, item attributes, price and logistics signals, content embeddings).
  \item $\mathbf{x}_H$ are \textbf{exposure-dependent historical belief signals}, meaning features derived from past exposure and interactions (e.g., stabilized CTR/CVR estimates, historical transactions, item IDs, or history-derived embeddings).
\end{itemize}

Historical belief signals $\mathbf{x}_H$ are powerful because they compactly summarize accumulated interaction evidence under past allocations.
However, they are also \emph{policy-mediated}: their distribution depends on the platform's prior exposure decisions.
As a result, for strategic segments such as cold-start or under-exposed inventory, $\mathbf{x}_H$ can differ substantially under the target population $\mathcal{Q}$, which upweights such segments, relative to the logging distribution $\mathcal{P}(\pi_t)$.
\vspace{-1mm}
\subsection{Over-Reliance on Endogenous Signals}
\label{sec:problem-twofeature}
To motivate semantics-aware control over learning dynamics, we first illustrate the pathology in a minimal setting.
Consider a linear predictor
$f_\mathbf{w}(\mathbf{x})=\mathbf{w}_C^\top \mathbf{x}_C + \mathbf{w}_H^\top \mathbf{x}_H$
trained to minimize expected loss under the logged distribution $\mathcal P$ (induced by deployment).
Suppose $\mathbf{x}_H$ is strongly predictive under $\mathcal P$ because it aggregates exposure-conditioned evidence and explains variance in selected logged contexts, while $\mathbf{x}_C$ is moderately predictive but broadly available.
If the learning algorithm observes $(\mathbf{x}_C,\mathbf{x}_H)$ jointly from the start, it will typically assign large weight to $\mathbf{x}_H$ because it yields the fastest reduction in empirical loss.
This can lead to \emph{over-reliance}:
\[
\|\mathbf{w}_H\| \text{ becomes large while } \|\mathbf{w}_C\| \text{ remains small.}
\]
Now consider a target population $\mathcal Q$ that deviates from $\mathcal P$, for example by upweighting the cold-start items.
In these segments, $\mathbf{x}_H$ is a weak predictor, so predictive performance depends disproportionately on the learned content pathway $\mathbf{w}_C$.
Thus, a model that is highly accurate under $\mathcal P$ can generalize poorly under $\mathcal Q$.

This effect is beyond merely `feature missingness'' in the target population and reflects a broader shift in the joint distribution $(c,i,y)$: the platform strategically cares about success events for items and contexts that are underrepresented under $\mathcal P(\pi_t)$.

The same phenomenon appears sharply in pairwise learning-to-rank.
Let $s_{\boldsymbol{\theta}}(c,i)$ be a scoring function and consider the standard pairwise logistic loss used in LambdaLoss-style training \cite{wang2018lambdaloss}:
\begin{equation}
\ell(\boldsymbol{\theta}; c,i,j)
\;=\;
\log\bigg(1+\exp\Big(-\big(s_{\boldsymbol{\theta}}(c,i)-s_{\boldsymbol{\theta}}(c,j)\big)\Big)\bigg).
\label{eq:pairwise-loss}
\end{equation}
For brevity, let $s_i := s_{\boldsymbol{\theta}}(c,i)$ and $s_j := s_{\boldsymbol{\theta}}(c,j)$.
The gradient magnitude with respect to the score difference is proportional to $\sigma\big(-(s_i-s_j)\big)$, where $\sigma$ is the sigmoid.

If exposure-dependent belief signals $\mathbf{x}_H$ provide a highly predictive shortcut, early training can quickly create large score gaps $|s_i-s_j|$ by exploiting $\mathbf{x}_H$.
Then $\sigma\big(-(s_i-s_j)\big)$ saturates toward $0$, and gradients for other parameters are correspondingly downweighted, including those that control content representations.
As a result, content signals $\mathbf{x}_C$ may be \emph{starved of learning signal} even if they are predictive and essential for strategic generalization.

This starvation mechanism is not tied to a particular hypothesis class.
It follows from loss geometry (saturation) and path dependence: once one feature group achieves separation early, the optimizer has little incentive to develop alternative separating representations.
While \emph{gradient starvation} is discussed in the context of over-parameterized neural networks under gradient descent updates \cite{pezeshki2021gradientstarvation}, we use the term more broadly here, with a particular focus on the semantics-driven case of historical belief versus content-based merit signals in ranking and prediction.
\vspace{-2mm}
\subsection{Formal objective and problem statement}
\label{sec:problem-formal}
Let $\mathcal F$ be a hypothesis class and let $f\in\mathcal F$ be a predictor, which induces a score-based allocation policy $\pi_f=\mathrm{argsort}(f)$.
Let $\mathcal L(f;\mathbf{z})$ be the task loss (e.g., MSE, cross-entropy, pairwise or listwise ranking loss).
For any distribution $\mathcal D$ over examples $\mathbf{z}$, define the population risk
\vspace{-2mm}
\[
\mathcal R_{\mathcal D}(f)\;\equiv\;\mathbb E_{\mathbf{z}\sim \mathcal D}\!\left[\mathcal L(f;\mathbf{z})\right].
\]
We observe data $\mathcal D_t\sim \mathcal P(\pi_t)$ but aim to minimize risk under a platform-chosen target distribution $\mathcal Q$.
Let $f^{\textsc{Full}}$ denote the model trained by standard joint empirical risk minimization (ERM) on all features $(\mathbf{x}_C,\mathbf{x}_H)$, and let $f^{\textsc{Content}}$ denote the content-only model trained with $\mathbf{x}_H$ masked.
Our goal is to design a training procedure that produces a predictor $f$ such that:
\begin{enumerate}
  \item \textbf{Target robustness:} $\mathcal R_{\mathcal Q}(f)$ improves relative to $f^{\textsc{Full}}$, especially on strategic subpopulations (e.g., cold-start or under-exposed inventory).
  \item \textbf{Source competitiveness:} $\mathcal R_{\mathcal P(\pi_t)}(f)$ remains competitive with $f^{\textsc{Full}}$ and does not collapse to $f^{\textsc{Content}}$.
  \item \textbf{Behavioral shaping / fairness:} $f$ reduces excessive dependence on $\mathbf{x}_H$, mitigating self-reinforcing exposure disparities in the induced allocation policy, while preserving overall marketplace KPIs.
\end{enumerate}

A clean and widely applicable instantiation is the \emph{frozen-start} target, where belief signals are unavailable at inference time:
\begin{equation}
\mathbf{x}_H \equiv 0 \quad \text{(or set to a default/prior) under } \mathcal Q.
\label{eq:frozen-start}
\end{equation}
This yields a controlled notion of distribution shift that isolates the central question: whether the learner has developed a predictive content pathway.
Although our motivating setting is ranking in dynamic marketplaces, the same problem arises in prediction tasks whenever a feature group is (i) highly predictive under the training distribution and (ii) unstable, missing, or strategically de-emphasized under the target distribution.
Our method applies to both ranking and pointwise prediction objectives.
\vspace{-1mm}
\section{Methodology: Representation Curriculum}
\label{sec:method}
We propose \emph{representation curriculum} (RC), a semantics-aware training time intervention that controls reliance on exposure-dependent historical belief signals via \emph{temporal staging}.
RC is motivated by the observation that when $\mathbf{x}_H$ is highly predictive under the logged distribution, standard training can over-invest in $\mathbf{x}_H$ early in optimization.
This yields a weak content pathway and poor robustness under distribution shift in the target population.

\paragraph{Feature staging via a gating schedule.}
RC introduces a \emph{feature-gating schedule} $m(t)\in[0,1]$ over training iterations $t$, defining the gated input
\begin{equation}
\tilde{\mathbf{x}}_t \;\equiv\; (\mathbf{x}_C,\; m(t)\,\mathbf{x}_H).
\label{eq:gating}
\end{equation}
While the schedule can be gradual, we focus on a practical two-stage curriculum:
\[
m(t)=0\;\;\text{for}\;\;t\le T_1
\qquad\text{and}\qquad
m(t)=1\;\;\text{for}\;\;t>T_1,
\]
that is, Stage~1 trains on content signals only, then Stage~2 trains on all signals.
Stage~1 produces a \emph{content anchor model} that must explain label variation using $\mathbf{x}_C$, and therefore learns content representations that are more robust to distribution shift. 

\paragraph{Anchored Stage~2 for a robust content pathway}
Let $f_{\boldsymbol{\theta}}$ denote the model class and let $\boldsymbol{\theta}_C$ denote the subset of parameters implementing the \emph{content pathway}
(e.g., a content tower in a two-tower neural model, or the Stage~1 component of an additive model).
Let $\boldsymbol{\theta}^{(1)}$ be the parameters obtained after Stage~1.
Stage~2 optimizes the task loss on full features while anchoring to the Stage~1 content pathway in one (or both) of the following ways.

\smallskip
\noindent\textbf{(i) Prediction anchoring (masked consistency).}
We enforce that the Stage~2 model remains consistent with the Stage~1 model on \emph{masked inputs}:
\begin{equation}
\mathcal L_{\text{pred}}(\boldsymbol{\theta})
\;=\;
\mathbb E_{(\mathbf{x},y)\sim\mathcal P(\pi_t)}
\Big[ d\Big(f_{\boldsymbol{\theta}}(\mathbf{x}_C,0),\, f_{\boldsymbol{\theta}^{(1)}}(\mathbf{x}_C,0)\Big)\Big],
\label{eq:pred-anchor}
\end{equation}
where $d(\cdot,\cdot)$ is a discrepancy in prediction space.
For regression we use squared error; for binary classification we use KL or squared logit differences; for ranking we use squared score
differences on query--item pairs (or KL on pairwise probabilities), consistent with the downstream surrogate loss.

\smallskip
\noindent\textbf{(ii) Parameter anchoring (content-path protection).}
We explicitly protect the content pathway parameters:
\begin{equation}
\mathcal L_{\text{par}}(\boldsymbol{\theta})
\;=\;
\|\boldsymbol{\theta}_C-\boldsymbol{\theta}_{C}^{(1)}\|_2^2.
\label{eq:param-anchor}
\end{equation}
This is useful when function-space consistency alone is insufficient (e.g., when $\mathbf{x}_H$ induces strong gradients that substantially
alter internal representations while keeping masked outputs approximately unchanged).

Prediction anchoring preserves function behavior under masked belief inputs, while parameter anchoring preserves internal content representations.
\smallskip
\noindent\textbf{Stage~2 objective.}
Let $\mathcal L(\cdot\,;\,y)$ be the primary supervised loss (MSE, log-loss, pairwise/listwise ranking surrogate).
RC trains Stage~2 with
\begin{equation}
\mathbb E_{(\mathbf{x},y)\sim\mathcal P_2}\!\left[
\mathcal L\!\left(f_{\boldsymbol{\theta}}(\mathbf{x}_C,\mathbf{x}_H);y\right)\right]
+\lambda\,\Omega(\boldsymbol{\theta})
+\mu_1\,\mathcal L_{\text{pred}}(\boldsymbol{\theta})
+\mu_2\,\mathcal L_{\text{par}}(\boldsymbol{\theta}),
\label{eq:stage2-obj}
\end{equation}
where $\Omega$ is standard regularization (weight decay, tree constraints, early stopping), and $(\mu_1,\mu_2)$ control anchoring strength.
By default we take $\mathcal P_2=\mathcal P(\pi_t)$ (the logged distribution), but RC naturally supports stage-specific data choices:
Stage~1 can be trained on a broader proxy $\mathcal P_1$ better aligned with learning content affinity (e.g., broader traffic),
while Stage~2 can use richer contexts where belief signals are informative (e.g., conversion-heavy traffic).

\begin{algorithm}[t]
\caption{Representation Curriculum (RC)}
\label{alg:rc}
\begin{algorithmic}[1]
\REQUIRE Stage-1 data $\mathcal D_1\!\sim\!\mathcal P_1$, Stage-2 data $\mathcal D_2\!\sim\!\mathcal P_2$; features $\mathbf{x}=(\mathbf{x}_C,\mathbf{x}_H)$;
loss $\mathcal L$; regularizer $\Omega$; anchors $(\mu_1,\mu_2)$.
\STATE \textbf{Stage 1 (content-only):} train $f_{\boldsymbol{\theta}^{(1)}}$ on $\{(\mathbf{x}_C,0),y\}\subset \mathcal D_1$ by minimizing
$\;\mathbb E[\mathcal L(f_\theta(\mathbf{x}_C,0);y)] + \lambda\Omega(\boldsymbol{\theta})$.
\STATE Initialize Stage 2 at $\boldsymbol{\theta}\leftarrow \boldsymbol{\theta}^{(1)}$.
\STATE \textbf{Stage 2 (full + anchors):} train $f_{\boldsymbol{\theta}}$ on $\{(\mathbf{x}_C,\mathbf{x}_H),y\}\subset \mathcal D_2$ by minimizing (\ref{eq:stage2-obj})
\RETURN $f_{\boldsymbol{\theta}}$.
\end{algorithmic}
\end{algorithm}

\paragraph{Key instantiations.}
RC is model-agnostic; the main design choice is what constitutes the \emph{content pathway} $\boldsymbol{\theta}_C$ and how anchoring is implemented.

\smallskip
\noindent\textbf{Neural models (two-tower + MLP head).}
In our recommender experiments, $f_{\boldsymbol{\theta}}$ is a two-tower model:
a user/context tower produces $u_\phi$, an item-content tower produces $v_\psi(\mathbf{x}_C)$, and the final score is computed by an MLP head
on the concatenation (and optionally explicit interactions),
\[
s_{\boldsymbol{\theta}}(\mathbf{x})=g_\omega\!\left(\big[u_\phi,\; v_\psi(\mathbf{x}_C),\; u_\phi\odot v_\psi(\mathbf{x}_C),\; \varphi(\mathbf{x}_H)\big]\right),
\]
where $\varphi(\mathbf{x}_H)$ embeds/scales belief features.
We define $\boldsymbol{\theta}_C$ as the parameters of the content tower (and, optionally, the content-dependent blocks of the head),
so $\mathcal L_{\text{par}}$ directly protects content representations.
Prediction anchoring $\mathcal L_{\text{pred}}$ is computed under masked belief input $(\mathbf{x}_C,0)$.
Practically, stability in Stage~2 is improved by initializing from $\boldsymbol{\theta}^{(1)}$ and using conservative optimization (e.g., not increasing
learning rate, optionally lowering it for $\boldsymbol{\theta}_C$).

\smallskip
\noindent\textbf{Gradient Boosting}
For Gradient Boosted Decision Trees (GBDTs), RC uses an integer $M$ as the stage boundary.
Stage~1 fits $M$ trees using only content features, yielding
\[
f^{(1)}(\mathbf{x}_C)=f_0+\sum_{m=1}^{M}\eta\,h_m(\mathbf{x}_C).
\]
Stage~2 continues boosting from this \emph{fixed} content pathway by using $f^{(1)}$ as an initial score / base margin and fitting the remaining
trees using all features:
\[
f^{\textsc{RC}}(\mathbf{x}_C,\mathbf{x}_H)=f^{(1)}(\mathbf{x}_C)+\sum_{m=M+1}^{T}\eta\,h_m(\mathbf{x}_C,\mathbf{x}_H).
\]
This realizes an especially strong form of parameter anchoring: the Stage~1 trees are never modified, so the content pathway cannot be overwritten.
The effective deviation introduced by belief features is controlled by $T-M$ (remaining capacity), shrinkage $\eta$, and early stopping.
When supported by the training framework, an additional masked-consistency term analogous to \eqref{eq:pred-anchor} can be implemented by adding
masked replicas of training instances with a pseudo-target equal to $f^{(1)}(\mathbf{x}_C)$; in our main GBDT experiments, tuning $M$ together with
standard boosting regularization is sufficient.

\paragraph{Relation to stochastic feature masking.}
Stochastic feature masking / dropout \cite{srivastava2014dropout} randomly zeros subsets of inputs throughout training.
While it can reduce reliance on any single feature, it is \emph{semantics-oblivious} and does not ensure a robust content pathway is learned
\emph{before} exposure-dependent shortcuts are introduced.
In contrast, RC explicitly stages access to $\mathbf{x}_H$ and then anchors the model to preserve content behavior and/or content parameters.
This targeted control is central to improving robustness on strategic populations where historic belief signals are not reliable predictors.

\paragraph{Evaluation and reliance diagnostics.}
RC aims to reduce over-reliance on exposure-dependent belief signals while maintaining overall utility.
Accordingly, we report (i) standard in-distribution metrics on held-out splits from $\mathcal P(\pi_t)$,
(ii) target robustness metrics on $\mathcal Q$---notably \emph{frozen-start} evaluation where belief signals are set to defaults,
and (iii) reliance diagnostics such as feature importance in trees, sensitivity/attribution measures in neural models, and
popularity/coverage metrics in ranking (Section~\ref{sec:experiments}).

\vspace{-1mm}
\section{Theoretical Guarantees}
\label{sec:theory}
We formalize Representation Curriculum in a linear-Gaussian setting with ridge regularization.
This yields closed-form characterizations that make the ``shape the update'' perspective precise and
quantify the tradeoff between performance on the logged distribution $\mathcal P(\pi_t)$ and a strategic target $\mathcal Q$.
Let $\mathbf{x}=(\mathbf{x}_C,\mathbf{x}_H)\in\mathbb R^{d_C+d_H}$ with zero mean.
We consider the population model
\[
y \;=\; \mathbf{x}_C^\top \beta_C + \mathbf{x}_H^\top \beta_H + \varepsilon,\qquad
\mathbb E[\varepsilon\mid \mathbf{x}]=0,\ \ \mathbb E[\varepsilon^2]=\sigma^2.
\]
Let $\Sigma_{\mathcal D}=\mathbb E_{\mathcal D}[\mathbf{x}\mathbf{x}^\top]$ denote the covariance under distribution $\mathcal D$,
and similarly $\Sigma_{\mathcal D,CC}=\mathbb E_{\mathcal D}[\mathbf{x}_C \mathbf{x}_C^\top]$. We compare three estimators trained on the logged distribution $\mathcal P$:

\paragraph{(i) Full ridge (no curriculum).}
\begin{equation}
\mathbf{w}^{\textsc{Full}}
\;=\;
\arg\min_{\mathbf{w}_C,\mathbf{w}_H}\;
\mathbb E_{\mathcal P}\!\left[(y-\mathbf{x}_C^\top \mathbf{w}_C-\mathbf{x}_H^\top \mathbf{w}_H)^2\right]
+\lambda(\|\mathbf{w}_C\|_2^2+\|\mathbf{w}_H\|_2^2).
\label{eq:full-ridge}
\end{equation}

\paragraph{(ii) Content-only ridge (Stage 1).}
\begin{equation}
\begin{aligned}
\mathbf{w}_C^{(1)}
=\arg\min_{\mathbf{w}_C}\;&
\mathbb E_{\mathcal P}\!\left[\big(y-\mathbf{x}_C^\top \mathbf{w}_C\big)^2\right]
+\lambda\|\mathbf{w}_C\|_2^2, \\
&\mathbf{w}^{\textsc{Content}} \equiv (\mathbf{w}_C^{(1)},0).
\end{aligned}
\label{eq:content-ridge}
\end{equation}

\paragraph{(iii) Representation curriculum (anchored ridge, Stage 2).}
We model Stage~2 as ridge with a \emph{content-parameter anchor}:
\begin{equation}
\begin{aligned}
\mathbf{w}^{\textsc{RC}}(\mu)
=\arg\min_{\mathbf{w}_C,\mathbf{w}_H}\;&
\mathbb E_{\mathcal P}\!\left[
\big(y-\mathbf{x}_C^\top \mathbf{w}_C-\mathbf{x}_H^\top \mathbf{w}_H\big)^2
\right] \\
&+ \lambda\|\mathbf{w}_C\|_2^2 + \lambda\|\mathbf{w}_H\|_2^2
+ \mu\|\mathbf{w}_C-\mathbf{w}_C^{(1)}\|_2^2 .
\end{aligned}
\label{eq:rc-ridge}
\end{equation}
This is the population analogue of Stage~2 parameter anchoring  in (\ref{eq:stage2-obj}).
Prediction anchoring yields closely related bounds; we defer that variant to Appendix~\ref{app: main-theorem}.

\subsection{Closed-form characterization}
\label{sec:theory-closedform}

Write the ridge normal equations under $\mathcal P$ using block covariance matrices:
\[
A_{CC} \!=\! \Sigma_{\mathcal P,CC}+\lambda I,\quad
A_{HH} \!=\! \Sigma_{\mathcal P,HH}+\lambda I,\quad
A_{CH}\!=\!\Sigma_{\mathcal P,CH},\quad A_{HC}\!=\!A_{CH}^\top.
\]
Define the Schur complement
\begin{equation}
A_0 \;\equiv\; A_{CC} - A_{CH}A_{HH}^{-1}A_{HC}.
\label{eq:schur}
\end{equation}
We assume $A_0\succ 0$ with eigenvalues in $[a_{\min},a_{\max}]$.
This condition fails only when content features are perfectly redundant once history is present (e.g., perfect collinearity).

\begin{lemma}[Content-path interpolation]
\label{lem:interp}
Define $\mathbf{w}_\mu=\mu(A_0+\mu I)^{-1}$. let $\mathbf{w}_C^{\textsc{Full}}$ denote the content block of $\mathbf{w}^{\textsc{Full}}$ and $\mathbf{w}_C^{(1)}$ denote the Stage~1 solution.
Then the content block of the RC solution satisfies
\begin{equation}
\mathbf{w}_C^{\textsc{RC}}(\mu)
\;=\;
\mathbf{w}_C^{\textsc{Full}} \;+\; \mu(A_0+\mu I)^{-1}\,\big(\mathbf{w}_C^{(1)}-\mathbf{w}_C^{\textsc{Full}}\big),
\label{eq:interp}
\end{equation}
and the history block satisfies
\begin{equation}
\mathbf{w}_H^{\textsc{RC}}(\mu)
\;=\;
\mathbf{w}_H^{\textsc{Full}}
\;-\;
A_{HH}^{-1}A_{HC}\,\big(\mathbf{w}_C^{\textsc{RC}}(\mu)-\mathbf{w}_C^{\textsc{Full}}\big).
\label{eq:interp-H}
\end{equation}
Moreover, $\mathbf{w}_\mu$ has eigenvalues $\frac{\mu}{a_i+\mu}\in[0,1)$, and thus
\begin{equation}
\begin{aligned}
\big\|\mathbf{w}_C^{\textsc{RC}}(\mu)-\mathbf{w}_C^{\textsc{Full}}\big\|
&\le
\frac{\mu}{a_{\min}+\mu}\,\|\Delta_C\|, \\
\big\|\mathbf{w}_C^{\textsc{RC}}(\mu)-\mathbf{w}_C^{(1)}\big\|
&\le
\frac{a_{\max}}{a_{\max}+\mu}\,\|\Delta_C\|.
\end{aligned}
\label{eq:disp-bounds1}
\end{equation}
where $\Delta_C \equiv \mathbf{w}_C^{(1)}-\mathbf{w}_C^{\textsc{Full}}$.
\end{lemma}

Lemma~\ref{lem:interp} makes explicit that $\mu$ traces a continuous path between
the full-model content coefficients ($\mu\!=\!0$) and the content-only coefficients ($\mu\!\to\!\infty$),
modulated by the conditioning of the Schur complement $A_0$.

We evaluate models on two distributions:
(i) the logged source $\mathcal P(\pi_t)$, and (ii) a strategic target $\mathcal Q$ that assigns higher mass to cold/under-exposed inventory.
To isolate the core cold-start shift in a clean form, we analyze the canonical target where
\begin{equation}
\mathbf{x}_H \equiv 0 \quad \text{a.s. under } \mathcal Q,
\label{eq:coldQ}
\end{equation}
i.e., historical belief signals are absent/unreliable in the target population.
This corresponds to ``frozen-start'' evaluation in our neural experiments and to the sparse-history segment in MSLR.

For squared loss under $\mathcal D$, the excess risk is
\(
\mathcal E_{\mathcal D}(\mathbf{w})=\mathbb E_{\mathcal D}[(\mathbf{x}^\top(\mathbf{w}-\beta))^2]
=(\mathbf{w}-\beta)^\top \Sigma_{\mathcal D}(\mathbf{w}-\beta).
\)

\subsection{Main theorem: Pareto tradeoff}
\label{sec:theory-main}

\begin{theorem}[Anchored ridge yields a quantified Pareto tradeoff]
\label{thm:pareto}
Assume $A_0\succ 0$ with eigenvalues in $[a_{\min},a_{\max}]$ and consider the cold target \eqref{eq:coldQ}.
Let $\Delta_C=\mathbf{w}_C^{(1)}-\mathbf{w}_C^{\textsc{Full}}$.
Define the source and target curvature constants
\[
L_{\mathcal P}\equiv \lambda_{\max}(\Sigma_{\mathcal P}),
\qquad
L_{\mathcal Q}\equiv \lambda_{\max}(\Sigma_{\mathcal Q,CC}),
\]
and the history-coupling factor
\(
\kappa_H \equiv 1 + \|A_{HH}^{-1}A_{HC}\|_{\text{op}}^2.
\)
Then for all $\mu\ge 0$ the following hold:

\paragraph{(A) Bounded regret to \textsc{Full} on the source $\mathcal P$.}
\begin{equation}
\big|\mathcal E_{\mathcal P}(\mathbf{w}^{\textsc{RC}}(\mu))-\mathcal E_{\mathcal P}(\mathbf{w}^{\textsc{Full}})\big|
\;\le\;
L_{\mathcal P}\,\kappa_H\left(\frac{\mu}{a_{\min}+\mu}\right)^2 \|\Delta_C\|^2 .
\label{eq:bound-P}
\end{equation}

\paragraph{(B) Bounded regret to \textsc{Content} on the target $\mathcal Q$.}
Since $\mathbf{x}_H\equiv 0$ under $\mathcal Q$, only the content block matters and
\begin{equation}
\big|\mathcal E_{\mathcal Q}(\mathbf{w}^{\textsc{RC}}(\mu))-\mathcal E_{\mathcal Q}(\mathbf{w}^{\textsc{Content}})\big|
\;\le\;
L_{\mathcal Q}\left(\frac{a_{\max}}{a_{\max}+\mu}\right)^2 \|\Delta_C\|^2 .
\label{eq:bound-Q}
\end{equation}

\paragraph{(C) Improvement over \textsc{Full} on $\mathcal Q$ for large enough $\mu$.}
If the content-only model is better than the full model on the target by margin
\(
m_{\mathcal Q}\equiv \mathcal E_{\mathcal Q}(\mathbf{w}^{\textsc{Full}})-\mathcal E_{\mathcal Q}(\mathbf{w}^{\textsc{Content}})>0,
\)
then any $\mu$ satisfying
\begin{equation}
L_{\mathcal Q}\left(\frac{a_{\max}}{a_{\max}+\mu}\right)^2 \|\Delta_C\|^2 \;\le\; m_{\mathcal Q},
\label{eq:mu-large}
\end{equation}
guarantees $\mathcal E_{\mathcal Q}(\mathbf{w}^{\textsc{RC}}(\mu)) \le \mathcal E_{\mathcal Q}(\mathbf{w}^{\textsc{Full}})$.

\paragraph{(D) Improvement over \textsc{Content} on $\mathcal P$ for small enough $\mu$.}
If the full model is better than the content-only model on the source by margin
\(
m_{\mathcal P}\equiv \mathcal E_{\mathcal P}(\mathbf{w}^{\textsc{Content}})-\mathcal E_{\mathcal P}(\mathbf{w}^{\textsc{Full}})>0,
\)
then any $\mu$ satisfying
\begin{equation}
L_{\mathcal P}\,\kappa_H\left(\frac{\mu}{a_{\min}+\mu}\right)^2 \|\Delta_C\|^2 \;\le\; m_{\mathcal P}
\label{eq:mu-small}
\end{equation}
guarantees $\mathcal E_{\mathcal P}(\mathbf{w}^{\textsc{RC}}(\mu)) \le \mathcal E_{\mathcal P}(\mathbf{w}^{\textsc{Content}})$.
\end{theorem}

\paragraph{Interpretation (how to read the constants).}
The theorem formalizes a \emph{Pareto tradeoff} between source and target risk as $\mu$ varies.
The conditioning of the Schur complement $A_0$ (via $a_{\min},a_{\max}$) governs how quickly the RC solution moves
from \textsc{Full} ($\mu\!=\!0$) toward \textsc{Content} ($\mu\!\to\!\infty$).
If content and history are perfectly collinear under $\mathcal P$, then $A_0\to 0$ and no curriculum can uniquely recover content effects.
The factor $\kappa_H$ captures how strongly history coefficients respond to changes in content coefficients
(via cross-covariance between $\mathbf{x}_C$ and $\mathbf{x}_H$).
Finally, the margins $m_{\mathcal P},m_{\mathcal Q}$ encode the empirical tension:
\textsc{Full} typically wins on $\mathcal P$ (history is predictive), while \textsc{Content} can win on $\mathcal Q$ (history is absent/unreliable).

\begin{corollary}[An explicit $\mu$-interval and a Pareto curve]
\label{cor:mu-interval}
Let $\Delta\equiv \|\Delta_C\|$.
If $m_{\mathcal Q}>0$ and $m_{\mathcal P}>0$, then any $\mu$ satisfying both \eqref{eq:mu-large} and \eqref{eq:mu-small}
simultaneously improves over \textsc{Full} on $\mathcal Q$ and over \textsc{Content} on $\mathcal P$.
Equivalently, the feasible interval is
\begin{equation}
\begin{aligned}
\big\|\mathbf{w}_C^{\textsc{RC}}(\mu)-\mathbf{w}_C^{\textsc{Full}}\big\|
&\le
\frac{\mu}{a_{\min}+\mu}\,\|\Delta_C\|, \\
\big\|\mathbf{w}_C^{\textsc{RC}}(\mu)-\mathbf{w}_C^{(1)}\big\|
&\le
\frac{a_{\max}}{a_{\max}+\mu}\,\|\Delta_C\|.
\end{aligned}
\label{eq:disp-bounds2}
\end{equation}
Moreover, \eqref{eq:bound-P}--\eqref{eq:bound-Q} quantify the Pareto curve:
as $\mu$ increases, target regret decays as $\big(\frac{a_{\max}}{a_{\max}+\mu}\big)^2$,
while source regret grows as $\big(\frac{\mu}{a_{\min}+\mu}\big)^2$.
\end{corollary}

Assumption \eqref{eq:coldQ} captures the cleanest cold-start shift (belief signals absent).
More generally, the target $\mathcal Q$ may differ from $\mathcal P$ in the joint distribution of
$(\text{context},\text{item},y)$ (e.g., the platform strategically upweights contexts where cold inventory is valuable),
and $\mathbf{x}_H$ may be present but less correlated with $y$ than under $\mathcal P$.
The same analysis extends by replacing $\Sigma_{\mathcal Q,CC}$ and the ``$\mathbf{x}_H\equiv 0$'' simplification
with the appropriate covariance structure; we provide the general form and proof details in Appendix~\ref{app: main-theorem}.

We additionally derive a conservative risk bound for GBDTs in Appendix~\ref{app: gbdt-bound}, showing that RC inherits the Stage~1 anchor's target performance up to a correction term controlled by Stage~2 capacity, yielding a verifiable sufficient condition for improving over a belief-dominant baseline on $\mathcal Q$.

\section{Experiments}
\label{sec:experiments}
We evaluate RC as a mechanism for controlling reliance on exposure-dependent historical belief signals in learning-to-rank (LTR) and response prediction.
We consider three representative experimental settings: (i) a public LTR benchmark with explicit behavioral features (MSLR-WEB10K), (ii) a public recommendation benchmark with endogenous interaction aggregates (MovieLens-20M), and (iii) a large-scale production A/B test in sponsored product search.
Across all settings, we test the same core hypothesis:
\emph{delaying access to historical belief signals during training improves robustness on cold and strategic segments while preserving strong overall performance.}

\subsection{MSLR-WEB10K: RC for LTR}
\label{sec:exp-mslr}
MSLR-WEB10K is a large-scale supervised learning-to-rank benchmark released by Microsoft Research \cite{mslrweb2010,qin2013letor4}.
Each instance is a query--URL pair with graded relevance in $\{0,1,2,3,4\}$ and a 136-dimensional feature vector; we follow the official five-fold splits.

\paragraph{Content-based merit vs.\ historical belief signals.}
A notable aspect of MSLR-WEB10K is that it includes explicit behavioral features alongside query--document match features.
Following the dataset documentation \cite{mslrweb2010}, we treat features $1$--$133$ as exposure-independent \emph{content-based merit} features and the last three features ($134$--$136$) as exposure-dependent \emph{historical belief} features derived from user interactions:
\begin{equation}
\mathbf{x}_{\merit} \equiv (x_1, \dots, x_{133}),
\qquad
\mathbf{x}_{\belief} \equiv (x_{134}, x_{135}, x_{136}).
\label{eq:mslr-split}
\end{equation}

\paragraph{Cold vs.\ warm query--URL pairs.}
Because MSLR-WEB10K does not provide impression logs, we operationalize ``cold'' using sparsity in behavioral features.
In our main experiments, we define a query--URL pair as \textbf{cold} if $x_{136}=0$ (zero dwell time), which yields a large cold segment: across the five folds, $74$--$75\%$ of test instances are cold.


\paragraph{Models and curriculum.}
We train LightGBM LambdaMART models with an NDCG objective using $2000$ trees, $63$ leaves, learning rate $0.05$, and row/feature subsampling of $0.8$.
For GBDTs, RC is implemented as a two-stage boosting schedule: we first fit $M$ trees using merit-only features (masking $\mathbf{x}_{\belief}$), then fit the remaining $2000{-}M$ trees using all features.
We sweep $M\in\{0,100,\dots,2000\}$ and report overall and cold NDCG@3 and AUC; results for other ranking metrics (NDCG@1, MAP@1, MAP@3) follow the same pattern.

\paragraph{Experimental protocol.}
All MSLR-WEB10K results are averaged over the official five folds with $200$ independent repeats per fold ($1000$ runs total); we report means and use significance testing at the $0.01$
level.    
\begin{figure}[t]
  \centering
  \vspace{-1mm}
  \includegraphics[width=\columnwidth]{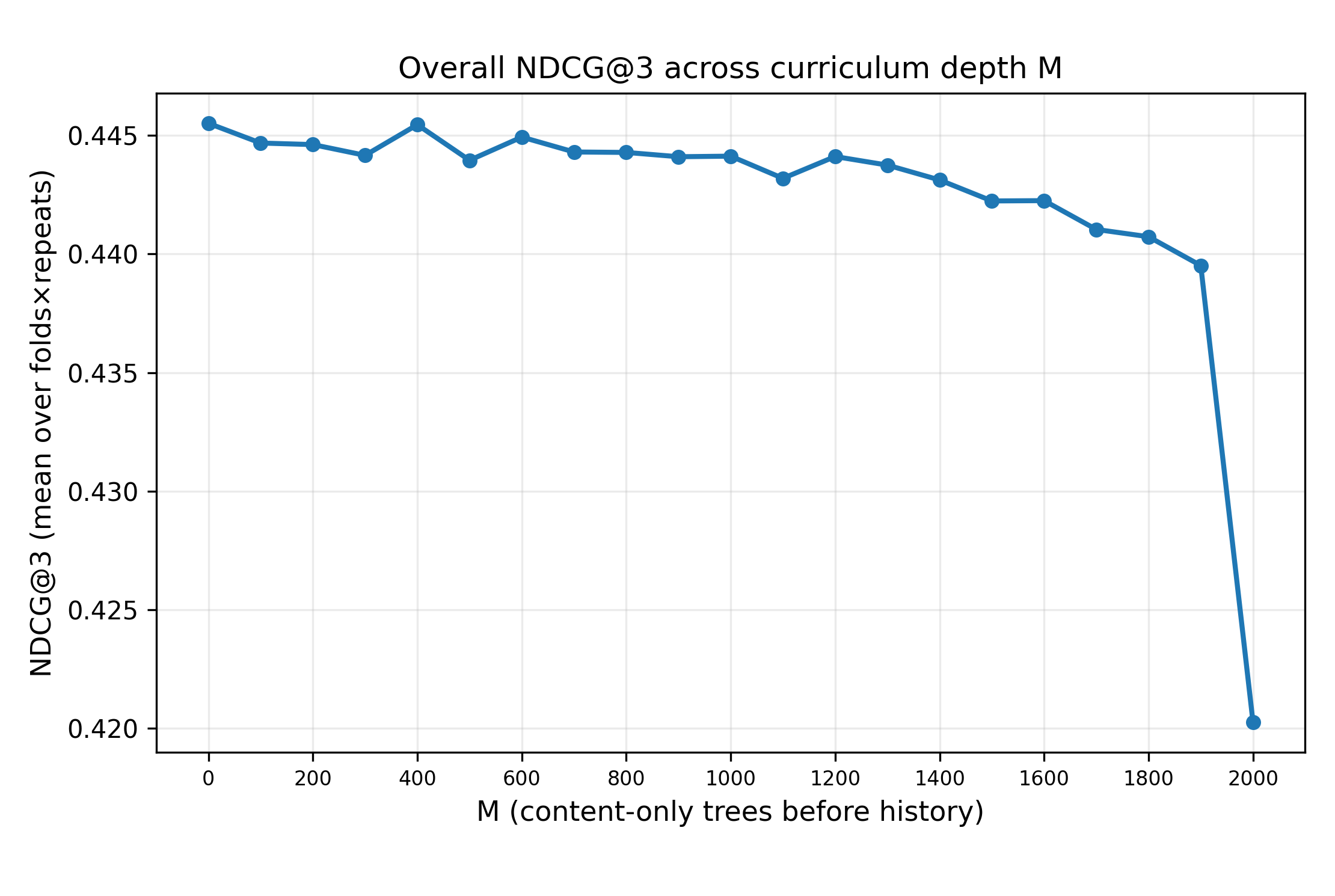}
  \vspace{-5mm}
  \caption{MSLR: overall NDCG@3 vs.\ curriculum depth $M$.}
  \label{fig:mslr-overall-ndcg3}
\end{figure}
\begin{figure}[t]
  \centering
  \vspace{-5mm}
  \includegraphics[width=\columnwidth]{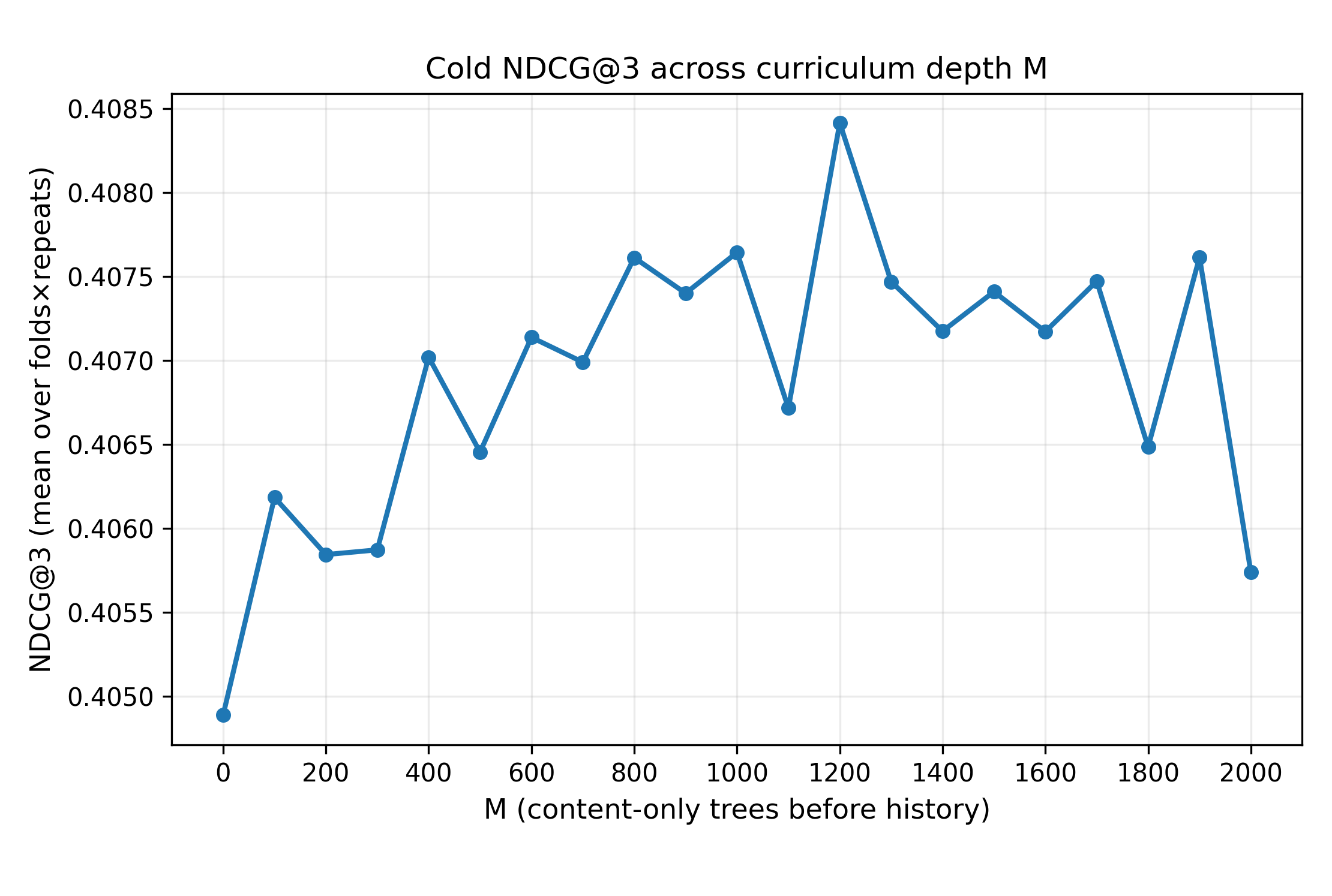}
  \vspace{-3mm}
  \caption{MSLR: cold NDCG@3 vs.\ curriculum depth $M$.}
  \label{fig:mslr-cold-ndcg3}
  \vspace{-3mm}
\end{figure}
\paragraph{Results and interpretation.}
RC consistently reduces reliance on belief features as $M$ increases; gain-based importance trends are reported in Figure~\ref{fig:mslr-importance}.
The NDCG curves exhibit the two-distribution tradeoff predicted by Theorem~\ref{thm:pareto}.
On cold instances (our empirical proxy for $\mathcal Q$, where belief signals are absent or attenuated), content-only training ($M{=}2000$) outperforms the full model ($M{=}0$) on average across five folds and five repeats, and intermediate curricula further improve cold NDCG (Figure~\ref{fig:mslr-cold-ndcg3}).
On the overall test distribution (proxy for $\mathcal P$), the full model outperforms the content-only baseline, and overall performance declines as $M\to 2000$ (Figure~\ref{fig:mslr-overall-ndcg3}).
Thus $M$ acts as a robustness--accuracy knob: compared to the full model, RC improves cold-start ranking while keeping overall performance close; compared to content-only training, RC preserves much higher overall quality while retaining most of the cold-start gains.
 
\paragraph{Reliance reduction via feature importance.}             Figure~\ref{fig:mslr-importance} tracks the gain-based importance of the historical belief features \texttt{F134}--\texttt{F136} as $M$ increases. In the full model ($M{=}0$), \texttt{F134} is by far the most important feature. Under RC its rank drops from $1$ at $M{=}0$ to $4$ at $M{=}100$, $8$ at $M{=}400$, and $14$ at $M{=}1000$, with the same monotone trend holding for \texttt{F135} and \texttt{F136}.                                               \begin{figure}[ht]                                                                 \centering                                                                 
    \includegraphics[width=\columnwidth]{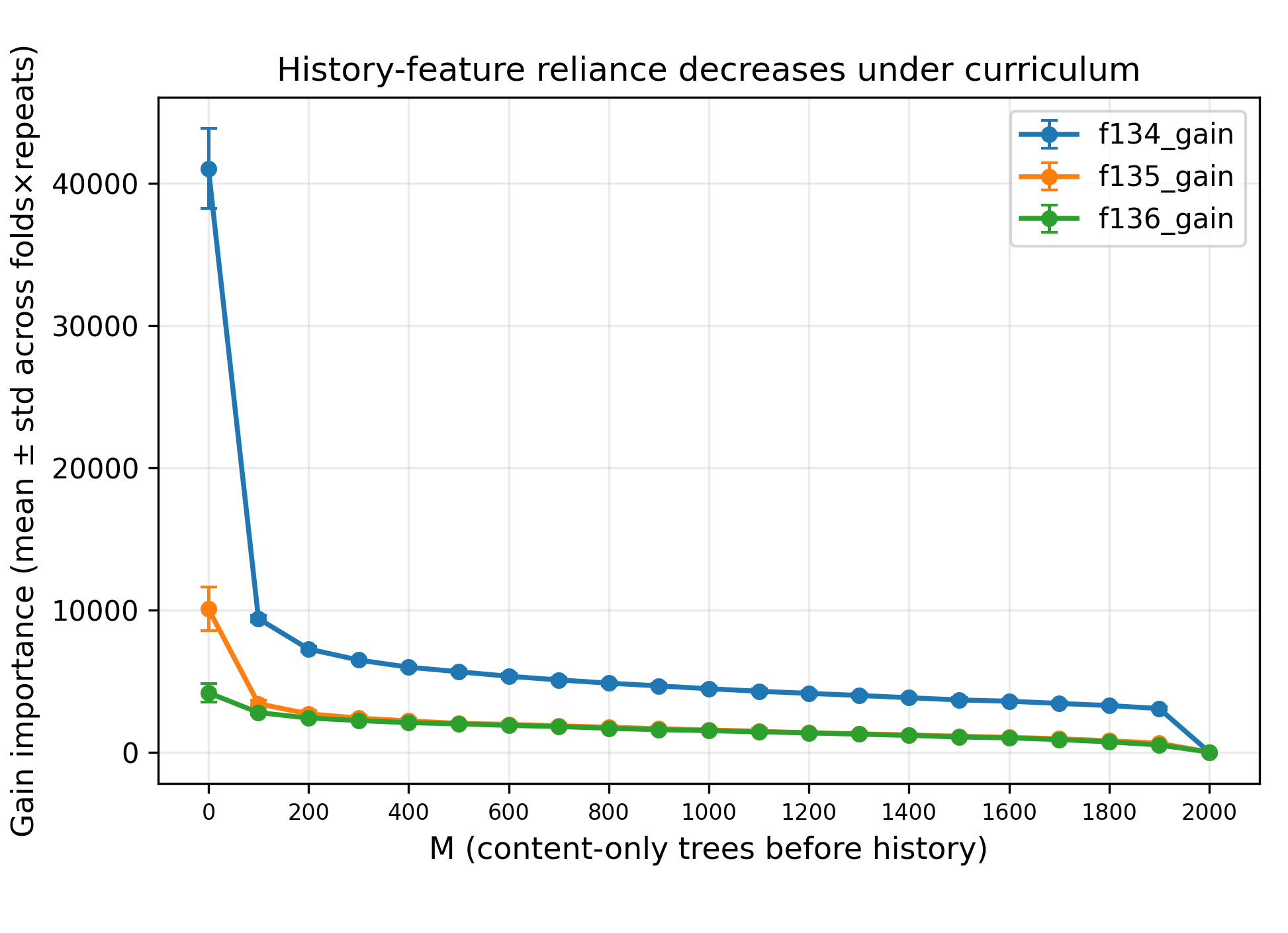}
    \caption{MSLR: gain-based importance of historical belief features \texttt{F134}--\texttt{F136} vs.\ $M$ (mean $\pm$ 1 std over 1000 runs).} 
    \label{fig:mslr-importance}                                              
  \end{figure}       

 \paragraph{Comparison with column subsampling (feature masking).}         Table~\ref{tab:colsample} shows that varying \texttt{colsample\_bytree} $\in\{0.6,0.7,0.8,0.9\}$ leaves the dominant feature gain in the $36$k--$44$k range and moves cold NDCG@3 only   
marginally ($0.4049$--$0.4053$), while overall NDCG@3 stays essentially at the full-model level ($0.4455$--$0.4456$). RC reduces the same feature's gain to $10{,}025$ ($M{=}100$) and     
$4{,}667$ ($M{=}1000$), lifting cold NDCG@3 to $0.4061$ and $0.4076$ respectively. The gap confirms that semantics-aware staging is the key driver: column subsampling (feature masking) 
does not distinguish content-based from history-based features and therefore cannot target shortcut reliance.                                                                              
\begin{table}[ht]              
\centering\small
\caption{RC vs.\ column subsampling (feature masking) on MSLR-WEB10K. RC achieves substantially stronger reliance reduction and cold-start gains.}                                         
\label{tab:colsample}            
\begin{tabular}{lccc}            
\toprule                         
Method & F134 gain & Cold NDCG@3 & Overall NDCG@3 \\                         
\midrule                                                                     
Full ($M{=}0$)     & 43,867          & 0.4049          & \textbf{0.4456} \\           
\texttt{col}$=0.9$ & 42,550          & 0.4053          & 0.4454 \\           
\texttt{col}$=0.8$ & 37,345          & 0.4051          & 0.4451 \\           
\texttt{col}$=0.7$ & 36,411          & 0.4049          & 0.4448 \\           
\midrule                                                                     
RC ($M{=}100$)     & \textbf{10,025} & \textbf{0.4061} & 0.4447 \\           
RC ($M{=}1000$)    & \textbf{4,667}  & \textbf{0.4076} & 0.4441 \\             \bottomrule                      
\end{tabular}                    
\end{table}         
\paragraph{AUC}                                        
Figures~\ref{fig:mslr-auc} and~\ref{fig:mslr-cold-auc} report overall and cold AUC. The general trend observed for NDCG holds for AUC as well: Representation Curriculum (RC) preserves    
strong overall AUC while improving cold AUC relative to the full model. Since the underlying LambdaMART base model is optimized for NDCG via a pairwise loss, AUC exhibits larger variance; accordingly, the full model can achieve higher overall AUC than the content-only model even on cold instances, which does not contradict our main findings.  
\begin{figure}[ht]                                                             \centering                                                                  
\includegraphics[width=\columnwidth]{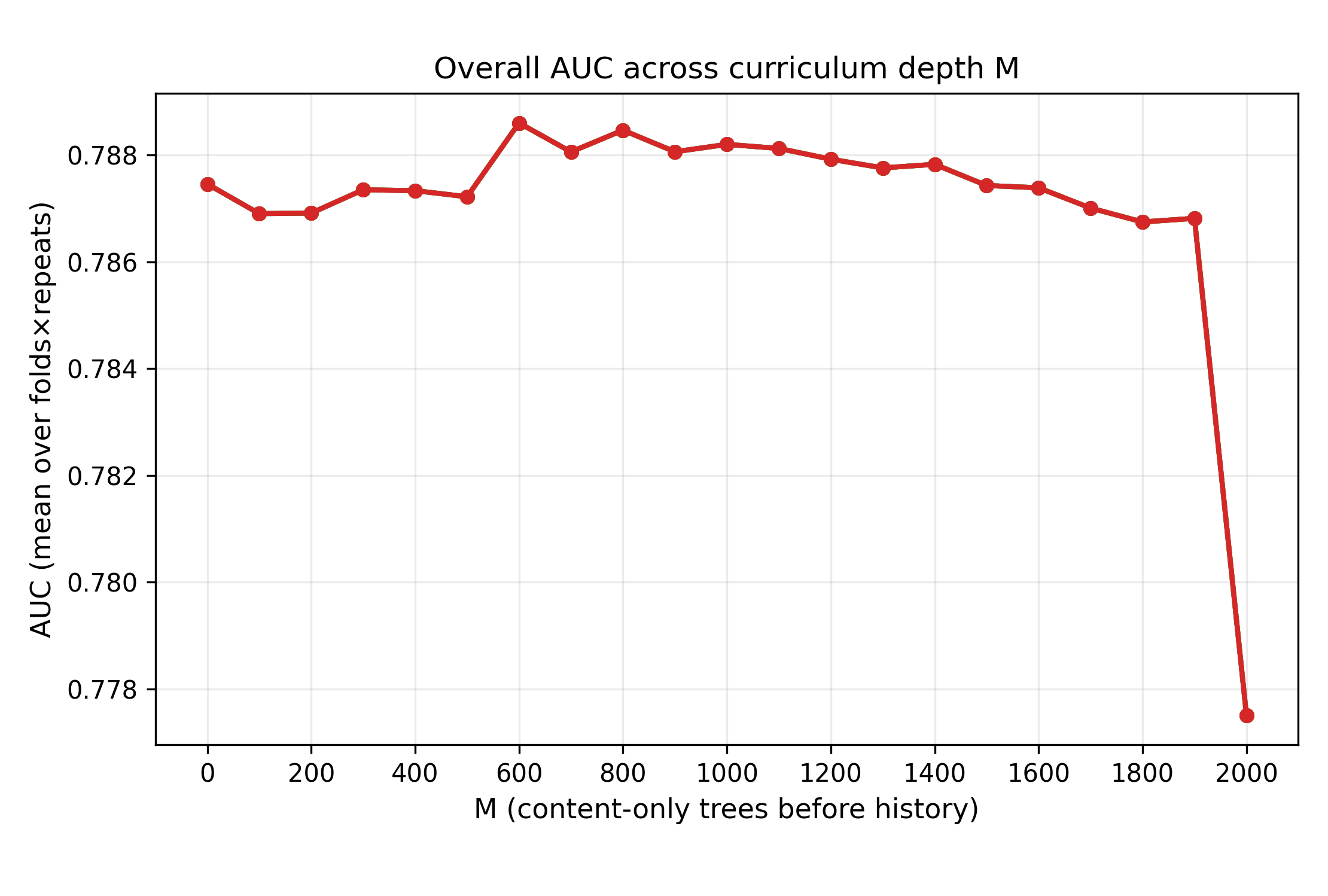}               
\caption{MSLR: overall AUC vs.\ curriculum depth $M$.}                    
\label{fig:mslr-auc}                                                      
\end{figure}                                                                
\begin{figure}[ht]                                                          
\centering                                                                
\includegraphics[width=\columnwidth]{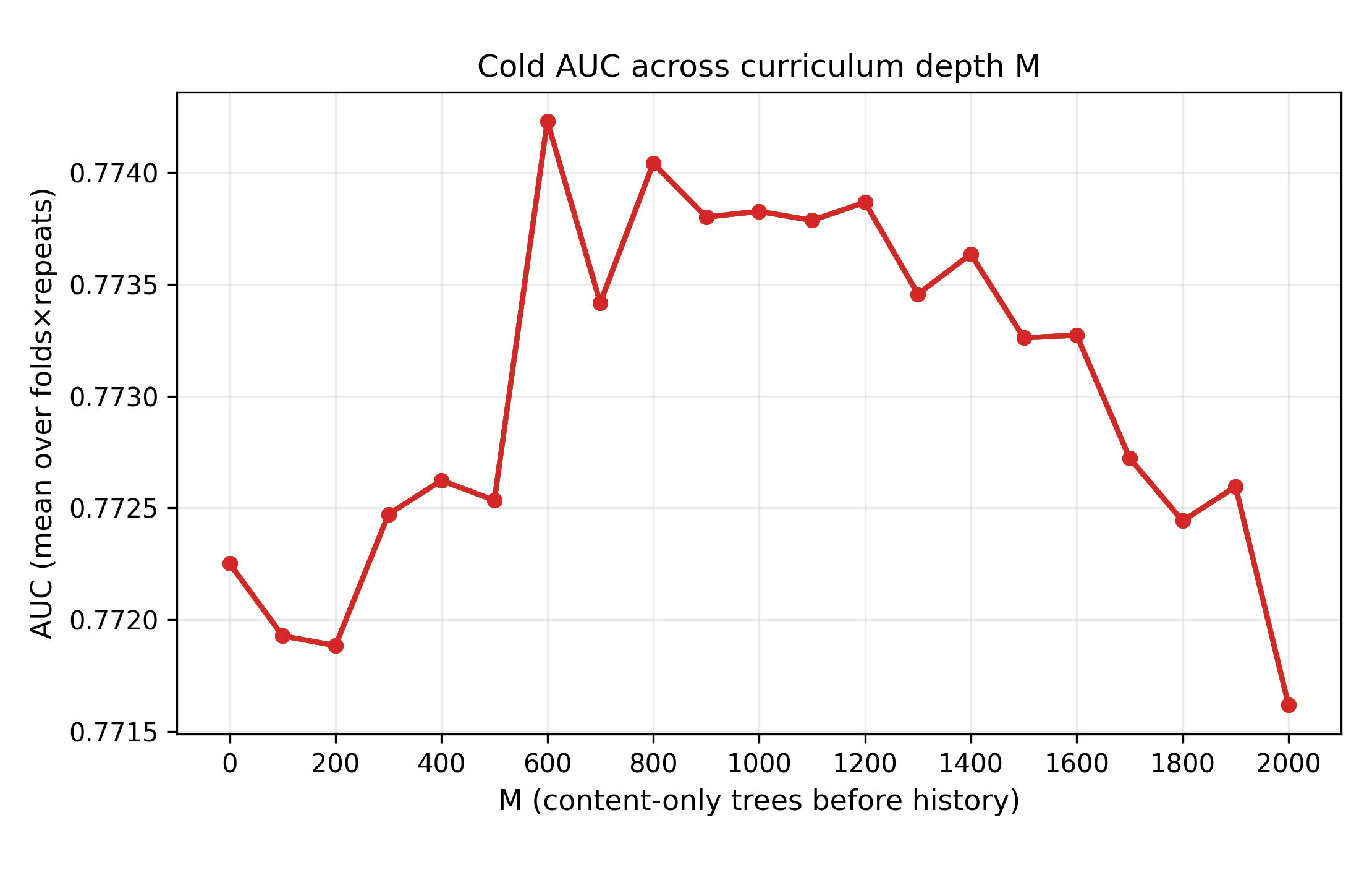}         
\caption{MSLR: cold AUC vs.\ curriculum depth $M$.}                       
\label{fig:mslr-cold-auc}                                                 
\end{figure}                              

\subsection{MovieLens-20M: RC in a neural recommendation setting}
\label{sec:exp-movielens}
We next evaluate representation curriculum on a public recommendation benchmark with rich \emph{item content} and naturally evolving \emph{historical interaction evidence}.
We use MovieLens-20M (timestamped ratings) and treat each interaction event $(u,i,t,r)$ as a supervised example.
While MovieLens is not collected under a controlled deployment policy, it provides a convenient testbed to (i) instantiate historical belief signals as \emph{endogenous, exposure-dependent aggregates} and (ii) study how curriculum changes reliance on such aggregates under controlled ``cold'' interventions.

\paragraph{Task and split.}
We use a \emph{time-based split} to respect the causality of history features: interactions are sorted by timestamp, with the last segment held out for test and a preceding segment for validation.
All history features for an event are computed strictly from interactions \emph{prior} to its timestamp.
We focus on \textbf{classification} with $y=\mathrm{1}\{r \ge 4\}$ (``like''), which is standard for implicit-feedback style evaluation and supports AUC and log-loss diagnostics.

\paragraph{Content-based merit signals (exposure-independent).}
We define item merit features from content text (movie title and genres).
Our main configuration uses sentence embeddings, particularly Sentence BERT\cite{reimers2019sentence} (\texttt{all-MiniLM-L6-v2}) reduced to dimension $d$ and concatenated with lightweight metadata (genres as tokens). We also considered representations based on TF-IDF and mean pooling of token embeddings.
On the user side, we consider two representations:
(i) \texttt{id} uses a user ID embedding;
(ii) \texttt{id+profile} augments the ID embedding with a \emph{train-only profile} computed by mean-pooling the content embeddings of items the user interacted with in training.
The profile captures stable preference structure without relying on item-level exposure aggregates at inference time.

\paragraph{Historical belief signals (exposure-dependent).}
We build item-level belief features using \emph{empirical-Bayes stabilized} estimates from interactions prior to time $t$.
For each event $(u,i,t)$, we form an EB posterior for the Bernoulli label (like/dislike) based on prior outcomes for item $i$ and extract: the posterior mean (belief), a clipped log-count (evidence strength), and an uncertainty proxy (e.g., posterior variance).
We use a content-based prior (nearest-neighbor on item content embeddings) with pseudo-count $\alpha$ (e.g., $\alpha=10$), mirroring how industrial systems stabilize sparse belief signals with content-conditioned priors.

\paragraph{Model class.}
We use the two-tower architecture with MLP head, discussed in section \ref{sec:method}.
Historical belief signals are injected through a dedicated scalar pathway (``history head'') that modulates the final logit.
This decomposition aligns with our semantic split: merit pathways are available for cold items, while belief pathways are only meaningful when sufficient interaction evidence exists.

\paragraph{Representation curriculum and anchoring.}
We train in two stages.
Stage~1 learns a merit-only model by masking historical belief signals, including item-ID-specific components that do not generalize to unseen items.
Stage~2 introduces belief signals and trains the full model while \emph{anchoring} to the Stage~1 solution.
We implement two complementary anchors: (i) a prediction-consistency term (distillation) that keeps the Stage~2 model close to the Stage~1 outputs under merit-only inputs, and (ii) a parameter anchor on the merit pathway parameters (towers and merit-related MLP blocks) to prevent Stage~2 from washing out learned merit structure.
We tie anchoring strengths and sweep $\mu_1,\mu_2 \in \{0, 0.01, 0.1\}$.

\paragraph{Evaluations}
MovieLens admits multiple ``cold'' notions (e.g., low training interaction count or low event-time evidence), but these slices can be brittle because missing ratings are not true non-exposures.
To make cold-start generalization \emph{causally interpretable} and robust, we report \textbf{frozen-start} evaluations that explicitly remove belief features at inference time:
\emph{(a)} \texttt{frozen-zero}: set belief features to their default/zero values and replace unseen item-ID components with OOV;
\emph{(b)} \texttt{frozen-prior}: set belief features to their EB prior values (no evidence).
These interventions isolate the contribution of learned merit structure from belief shortcuts. We report overall and Frozen start AUC as key evaluation metrics.

\paragraph{Results summary.}
Given consistent observations across multiple configurations, we report results from a representative setting and differ full report on ablations to subsequent reproducibility studies. 
The \textsc{Full} model achieves the best in-distribution test AUC ($0.678$) but collapses under frozen start, dropping to AUC $0.580$ when belief features are removed at inference.
A two-stage curriculum \emph{without} anchoring already improves robustness via a warm-start effect: initializing Stage~2 from the merit-only anchor (RC with $\mu_1=\mu_2=0$) raises frozen-start AUC to $0.620$ (with test AUC $0.669$).
Prediction anchoring is the main knob that makes frozen-start behavior match the content model: RC with $\mu_1=0.1$ and $\mu_2=0$ reaches frozen-start AUC $0.661$, essentially matching the content-only baseline (AUC $0.662$), while keeping test AUC $0.664$.
This is the intended behavior-shaping effect: Stage~2 leverages historical belief signals where available, while preserving a high-quality merit pathway that remains predictive when belief is absent. 

\subsection{Production A/B Test: Sponsored Product Search in a Dynamic Marketplace}
\label{sec:abtest}
We evaluate representation curriculum (RC) in a production sponsored product search system at a major e-commerce marketplace.
Following the marketplace-evaluation abstraction of \cite{ebrahimzadeh2024mev}, we view the ranker as an \emph{allocation mechanism} that assigns exposure to items, and measure relevant exposure and utility metrics that are measurable via short horizon randomized experiments.
Our goal is to test the paper's central claim: \emph{semantics-aware control of the optimization trajectory}, implemented by delaying access to a highly predictive exposure-dependent signal, reduces over-reliance on historic belief signals, increases exposure for items with limited historical evidence, and improves outcomes on strategic cold inventory segments, while keeping aggregate marketplace KPIs neutral.

\paragraph{Policy, baseline, and treatment.}
The baseline is a standard ``full'' ranker trained on the complete feature set, including both exposure-independent \emph{merit} signals and exposure-dependent \emph{belief} signals (history- and ID-derived signals), with no curriculum.
The RC treatment uses the same training data, objective, and model capacity as the baseline, but alters the training \emph{trajectory} by masking a single highly predictive \emph{ID-based belief feature} during the first training stage.
This feature acts as a strong shortcut because it summarizes item-specific historic exposure and observed transaction evidence and collapses to a prior when an item has no historical purchases; therefore, it disproportionately advantages items with established exposure and conversions.
In the second stage, we unmask this feature and continue training with the full feature set.
(For confidentiality, we omit platform-specific details such as the exact feature definition and thresholds.)

\paragraph{Offline diagnostics (behavioral effect of RC)}
Before online experimentation, we validate that RC changes feature reliance in the intended direction.
Compared to the baseline, RC reduces the measured importance of the masked belief feature by more than $70\%$ (gain-based importance), with no material degradation in overall offline ranking metrics.
On an offline cold slice defined by items with historical impressions below a threshold $\tau$, RC yields a lift of $>1\%$ in offline utility metrics, motivating the online test.

\paragraph{Online experiment design and metrics.}
We run a controlled randomized A/B test over production traffic for approximately two weeks, comparing the RC policy to the baseline.
We report \emph{relative lifts} in (i) exposure metrics and (ii) outcome metrics, evaluated on stratified inventory segments designed to isolate cold and strategic supply.
Exposure is measured as the number of times an item is shown in search results (aggregated over placements), while outcomes are measured via \emph{sale velocity}: the fraction (or count) of listings that obtain at least one transaction within a fixed window of $Y$ days from listing/publish time.\footnote{We use sale velocity because it is sensitive to marketplace coverage and newcomer success, and it captures a long-run objective beyond per-impression click/purchase rates.}
We measured aggregate marketplace KPIs (e.g., transactions and revenue), which remained neutral over the test window.

\paragraph{Results and Discussions}
Tables~\ref{tab:exposure-newItems} and~\ref{tab:Salevelocity-NewItems} summarize the primary observed effects in the AB test.
The RC policy increases exposure for items with weak purchase history, particularly newly listed inventory, while preserving overall business outcomes.

\begin{table}[t]
\centering
\begin{tabular}{@{}lc@{}}
\toprule
Inventory segment & Search exposure lift \\
\midrule
No historic impressions & neutral \\
No historic purchases & $+0.33\%$ \\
No historic purchases \& listing age $< X$ & $+0.50\%$ \\
\bottomrule
\end{tabular}
\caption{Relative lift in exposure under RC vs.\ baseline for cold inventory segments.}
\label{tab:exposure-newItems}
\end{table}

\begin{table}[t]
\centering
\begin{tabular}{@{}lc@{}}
\toprule
Items converted within $Y$ days & Transaction-volume lift \\
\midrule
Total & $+2.25\%$ \\
Listing age $< X$ & $+1.31\%$ \\
Listed by new listers & $+4.35\%$ \\
\bottomrule
\end{tabular}
\caption{Relative lift in sale velocity under RC vs.\ baseline on strategic inventory segments.}
\label{tab:Salevelocity-NewItems}
\end{table}
We make a few remarks to help interpret the magnitude and direction of the observed effects.
First, the intervention masks only one strong ID-based belief feature; other historic belief signals, including uncertainty-aware historical estimates (e.g., empirical-Bayes style aggregates), remain available, which can attenuate shifts in the most extreme ``no-impression'' segments but still meaningfully affects ``no-purchase'' segments where belief features can otherwise dominate.
Second, exposure is reported across all placements while the treatment changes only sponsored ranking; the resulting reported lifts are directionally diluted, and conservative.
Third, we observe that exposure increases not only for cold items but also for items with very large historical transaction volume (above a threshold $Z$), alongside reduced exposure for moderately established items; qualitatively, RC appears to reduce reliance on a single strong shortcut and reallocates exposure away from a narrow band of historically advantaged inventory.
Finally, while aggregate marketplace KPIs remain neutral, sale velocity improves, aligning with a mechanism-design objective in dynamic marketplaces: increase the probability that new and under-exposed listings earn early exposure and can demonstrate merit through realized transactions.
Some segment-level outcome lifts are directionally positive but estimated with larger uncertainty due to smaller subpopulation mass; we therefore emphasize the robust pattern across exposure and sale-velocity metrics rather than over-interpreting any single fine-grained slice.
\subsection{Reproducibility.}
Code, notebooks, and experiment configurations for MSLR-WEB10K and MovieLens-20M are available at \url{https://github.com/sinaBaharlouei/CurriculumRepresentation}, including scripts to reproduce the reported metrics and figures.
Due to privacy and platform constraints, production data and code are not released; we report aggregated A/B results and provide sufficient implementation details to replicate RC on analogous marketplace data.

\section{Conclusion}
\label{sec:conclusion}
Exposure-dependent historic belief signals are powerful predictors in ranking and recommendation, but over-reliance on them can yield brittle generalization on strategically important sub-populations. We introduced \textbf{Representation Curriculum} (RC), a semantics-aware optimization-trajectory intervention that trains a content pathway first (masking belief signals) and then enables all signals while anchoring the learned content representation. Our linear analysis provides a closed-form characterization and a quantified Pareto tradeoff between performance on the logged source distribution and robustness in cold start regimes. Empirically, we show across multiple settings that RC reduces reliance on historic features and improves the performance cold-segment without material impact on overall performance.

Future work includes (i) combining feature curricula with standard instance-based curricula by choosing stage-specific training distributions to learn intent affinity/content-based merit broadly before specializing with belief signals on (converting) contexts where belief signals are most informative (ii) extending guarantees beyond frozen start and (iii) advancing a mechanism-design view of \emph{controllable} allocation policies by developing explicit knobs and diagnostics for policy properties, such as exposure elasticity to popularity, coverage targets, and amortized exposure constraints aligned with marketplace welfare.
\newpage
\bibliographystyle{ACM-Reference-Format}
\bibliography{sample-base}

\appendix

\section{Proofs for Section~\ref{sec:theory}}
\label{app: main-theorem}
This appendix provides full proofs for Lemma~\ref{lem:interp}, Theorem~\ref{thm:pareto}, and
Corollary~\ref{cor:mu-interval}. We follow the notation of Section~\ref{sec:theory}:

$x=(x_C,x_H)\in\mathbb{R}^{d_C+d_H}$, $y=x_C^\top\beta_C+x_H^\top\beta_H+\varepsilon$ with
$\mathbb{E}[\varepsilon\mid x]=0$, and $\Sigma_{\mathcal D}=\mathbb{E}_{\mathcal D}[xx^\top]$.

\subsection{Population Normal Equation}
For any distribution $\mathcal D$ and any $w=(w_C,w_H)$, define the squared-loss risk
$R_{\mathcal D}(w)=\mathbb{E}_{\mathcal D}[(y-x_C^\top w_C-x_H^\top w_H)^2]$.
Using the model assumptions and $\mathbb{E}[x\varepsilon]=0$, we have
\begin{equation}
R_{\mathcal D}(w)=\sigma^2 + (w-\beta)^\top \Sigma_{\mathcal D}(w-\beta).
\label{eq:risk-decomp}
\end{equation}
Particularly, the \emph{excess risk} in Section~\ref{sec:theory} is
$\mathcal{E}_{\mathcal D}(w)=(w-\beta)^\top \Sigma_{\mathcal D}(w-\beta)$.

\paragraph{Full ridge under $\mathcal P$.}
The objective in \eqref{eq:full-ridge} is $R_{\mathcal P}(w)+\lambda\|w\|_2^2$.
Differentiating and setting the gradient to zero yields the population normal equations
\begin{equation}
(\Sigma_{\mathcal P}+\lambda I)\,w^{\textsc{Full}} \;=\; \mathbb{E}_{\mathcal P}[xy]
\;=\;\mathbb{E}_{\mathcal P}[xx^\top]\beta \;=\;\Sigma_{\mathcal P}\beta.
\label{eq:full-normal}
\end{equation}

\paragraph{Content-only ridge (Stage 1) under $\mathcal P$.}
Similarly, differentiating \eqref{eq:content-ridge} gives
\begin{equation}
(\Sigma_{\mathcal P,CC}+\lambda I)\,w_C^{(1)} \;=\; \mathbb{E}_{\mathcal P}[x_Cy]
\;=\; \Sigma_{\mathcal P,CC}\beta_C + \Sigma_{\mathcal P,CH}\beta_H.
\label{eq:stage1-normal}
\end{equation}

\paragraph{Anchored ridge (Stage 2) under $\mathcal P$.}
The objective in \eqref{eq:rc-ridge} is
$R_{\mathcal P}(w)+\lambda(\|w_C\|^2+\|w_H\|^2)+\mu\|w_C-w_C^{(1)}\|^2$.
Differentiating and setting the gradient to zero yields
\begin{equation}
(\Sigma_{\mathcal P}+\lambda I)\,w^{\textsc{RC}}(\mu) + \mu \begin{bmatrix} w_C^{\textsc{RC}}(\mu) \\ 0\end{bmatrix}
\;=\;
\Sigma_{\mathcal P}\beta + \mu \begin{bmatrix} w_C^{(1)} \\ 0\end{bmatrix}.
\label{eq:stage2-normal}
\end{equation}

\subsection{Proof of Lemma~\ref{lem:interp} (interpolation)}
\label{app:proof-interp}

\begin{proof}[Proof of Lemma~\ref{lem:interp}]
Write the (regularized) block covariances under $\mathcal P$ as in Section~\ref{sec:theory-closedform}:
\[
A_{CC}=\Sigma_{\mathcal P,CC}+\lambda I,\quad
A_{HH}=\Sigma_{\mathcal P,HH}+\lambda I,\quad
A_{CH}=\Sigma_{\mathcal P,CH},\quad A_{HC}=A_{CH}^\top.
\]
Let $b=\mathbb{E}_{\mathcal P}[xy]=\Sigma_{\mathcal P}\beta$ and write its blocks as
$b_C=\mathbb{E}_{\mathcal P}[x_Cy]$ and $b_H=\mathbb{E}_{\mathcal P}[x_Hy]$.
Then \eqref{eq:full-normal} is equivalent to the block system
\begin{equation}
\begin{aligned}
A_{CC}\,w_C^{\textsc{Full}} + A_{CH}\,w_H^{\textsc{Full}} &= b_C,\\
A_{HC}\,w_C^{\textsc{Full}} + A_{HH}\,w_H^{\textsc{Full}} &= b_H.
\end{aligned}
\label{eq:block-full}
\end{equation}
Since $A_{HH}\succ 0$, we can eliminate $w_H^{\textsc{Full}}=A_{HH}^{-1}(b_H-A_{HC}w_C^{\textsc{Full}})$ and obtain
\begin{equation}
A_0\,w_C^{\textsc{Full}} \;=\; b_C - A_{CH}A_{HH}^{-1}b_H,
\label{eq:full-schur}
\end{equation}
where $A_0=A_{CC}-A_{CH}A_{HH}^{-1}A_{HC}$ is the Schur complement \eqref{eq:schur}.

Next, the anchored system \eqref{eq:stage2-normal} can be written in blocks as
\begin{equation}
\begin{aligned}
(A_{CC}+\mu I)\,w_C^{\textsc{RC}}(\mu) + A_{CH}\,w_H^{\textsc{RC}}(\mu) &= b_C + \mu w_C^{(1)},\\
A_{HC}\,w_C^{\textsc{RC}}(\mu) + A_{HH}\,w_H^{\textsc{RC}}(\mu) &= b_H.
\end{aligned}
\label{eq:block-rc}
\end{equation}
Eliminating $w_H^{\textsc{RC}}(\mu)=A_{HH}^{-1}(b_H-A_{HC}w_C^{\textsc{RC}}(\mu))$ gives
\begin{equation}
(A_0+\mu I)\,w_C^{\textsc{RC}}(\mu)
\;=\;
b_C + \mu w_C^{(1)} - A_{CH}A_{HH}^{-1}b_H.
\label{eq:rc-schur}
\end{equation}
Comparing \eqref{eq:rc-schur} with \eqref{eq:full-schur}, we have
$b_C - A_{CH}A_{HH}^{-1}b_H = A_0 w_C^{\textsc{Full}}$, hence
\begin{equation}
(A_0+\mu I)\,w_C^{\textsc{RC}}(\mu) = A_0 w_C^{\textsc{Full}} + \mu w_C^{(1)}.
\end{equation}
Multiplying by $(A_0+\mu I)^{-1}$ yields
\[
w_C^{\textsc{RC}}(\mu)
= (A_0+\mu I)^{-1}A_0\,w_C^{\textsc{Full}} + (A_0+\mu I)^{-1}\mu w_C^{(1)}.
\]
Using $(A_0+\mu I)^{-1}A_0 = I-\mu(A_0+\mu I)^{-1}$, we obtain the interpolation identity
\[
w_C^{\textsc{RC}}(\mu)
= w_C^{\textsc{Full}} + \mu(A_0+\mu I)^{-1}\big(w_C^{(1)}-w_C^{\textsc{Full}}\big),
\]
which is \eqref{eq:interp} with $W_\mu=\mu(A_0+\mu I)^{-1}$.

For the history block, from \eqref{eq:block-rc} and \eqref{eq:block-full} we have
\[
w_H^{\textsc{RC}}(\mu)=A_{HH}^{-1}\!\big(b_H-A_{HC}w_C^{\textsc{RC}}(\mu)\big),
\qquad
w_H^{\textsc{Full}}=A_{HH}^{-1}\!\big(b_H-A_{HC}w_C^{\textsc{Full}}\big),
\]
and subtracting gives \eqref{eq:interp-H}:
\[
w_H^{\textsc{RC}}(\mu)-w_H^{\textsc{Full}} = -A_{HH}^{-1}A_{HC}\big(w_C^{\textsc{RC}}(\mu)-w_C^{\textsc{Full}}\big).
\]

Finally, since $A_0$ is symmetric with eigenvalues $\{a_i\}$, $W_\mu$ is a matrix function of $A_0$:
$W_\mu=f_\mu(A_0)$ with $f_\mu(a)=\mu/(a+\mu)$. Therefore $W_\mu$ has eigenvalues $f_\mu(a_i)\in[0,1)$.
Its operator norm is $\|W_\mu\|_{\mathrm{op}}=\max_i \mu/(a_i+\mu)=\mu/(a_{\min}+\mu)$, which implies
\[
\|w_C^{\textsc{RC}}(\mu)-w_C^{\textsc{Full}}\|
=\|W_\mu\Delta_C\|
\le \|W_\mu\|_{\mathrm{op}}\|\Delta_C\|
\le \frac{\mu}{a_{\min}+\mu}\|\Delta_C\|.
\]
Moreover,
$w_C^{\textsc{RC}}(\mu)-w_C^{(1)} = (W_\mu-I)\Delta_C = -A_0(A_0+\mu I)^{-1}\Delta_C$ and
$\|A_0(A_0+\mu I)^{-1}\|_{\mathrm{op}}=\max_i a_i/(a_i+\mu)=a_{\max}/(a_{\max}+\mu)$, giving the second inequality
in \eqref{eq:disp-bounds2}.
\end{proof}

\subsection{An Auxiliary Inequality: Excess Risk vs.\ Model Disagreement}
\label{app:risk-disagreement}

Theorem~\ref{thm:pareto} controls how far $w^{\textsc{RC}}(\mu)$ moves from $w^{\textsc{Full}}$ (on $\mathcal P$)
and from $w^{\textsc{Content}}$ (on $\mathcal Q$). A convenient quantity for this is the \emph{prediction disagreement}
\begin{equation}
\mathcal{D}_{\mathcal D}(u,v)\;:=\;\mathbb{E}_{\mathcal D}\!\left[(x^\top(u-v))^2\right]
\;=\;(u-v)^\top \Sigma_{\mathcal D}(u-v),
\label{eq:disagreement}
\end{equation}
which is exactly the expected squared change in predictions between $u$ and $v$ under $\mathcal D$.

\begin{lemma}[From disagreement to excess-risk regret]
\label{lem:disagreement-to-risk}
For any distribution $\mathcal D$ and any $u,v\in\mathbb{R}^{d_C+d_H}$,
\begin{equation}
\big|\sqrt{\mathcal{E}_{\mathcal D}(u)}-\sqrt{\mathcal{E}_{\mathcal D}(v)}\big|
\;\le\;\sqrt{\mathcal{D}_{\mathcal D}(u,v)}.
\label{eq:sqrt-lipschitz}
\end{equation}
Consequently,
\begin{equation}
\mathcal{E}_{\mathcal D}(u)
\;\le\;
\Big(\sqrt{\mathcal{E}_{\mathcal D}(v)}+\sqrt{\mathcal{D}_{\mathcal D}(u,v)}\Big)^2,
\label{eq:risk-upper}
\end{equation}
and symmetrically with $u$ and $v$ swapped. In particular,
\begin{equation}
\big|\mathcal{E}_{\mathcal D}(u)-\mathcal{E}_{\mathcal D}(v)\big|
\;\le\;
2\sqrt{\mathcal{E}_{\mathcal D}(v)}\,\sqrt{\mathcal{D}_{\mathcal D}(u,v)}
+\mathcal{D}_{\mathcal D}(u,v).
\label{eq:risk-diff-bound}
\end{equation}
\end{lemma}

\begin{proof}
Let $S_{\mathcal D}=\Sigma_{\mathcal D}^{1/2}$ be the symmetric square root.
Then $\sqrt{\mathcal{E}_{\mathcal D}(u)}=\|S_{\mathcal D}(u-\beta)\|_2$ and
$\sqrt{\mathcal{D}_{\mathcal D}(u,v)}=\|S_{\mathcal D}(u-v)\|_2$.
By the reverse triangle inequality,
\[
\big|\|S_{\mathcal D}(u-\beta)\|_2-\|S_{\mathcal D}(v-\beta)\|_2\big|
\le \|S_{\mathcal D}(u-v)\|_2,
\]
which proves \eqref{eq:sqrt-lipschitz}. Inequality \eqref{eq:risk-upper} follows from the (forward) triangle inequality:
$\|S(u-\beta)\|\le \|S(v-\beta)\|+\|S(u-v)\|$ and squaring both sides. Finally,
\eqref{eq:risk-diff-bound} follows from \eqref{eq:risk-upper} by expanding the square.
\end{proof}

\subsection{Proof of Theorem~\ref{thm:pareto}}
\label{app:proof-pareto}

\begin{proof}[Proof of Theorem~\ref{thm:pareto}]
We prove each item in turn, using Lemma~\ref{lem:interp} and the disagreement quantity \eqref{eq:disagreement}.

\paragraph{Step 1: A bound on the content displacement.}
Lemma~\ref{lem:interp} gives
$w_C^{\textsc{RC}}(\mu)-w_C^{\textsc{Full}} = W_\mu\Delta_C$ with
$\|W_\mu\|_{\mathrm{op}}\le \mu/(a_{\min}+\mu)$. Hence
\begin{equation}
\|w_C^{\textsc{RC}}(\mu)-w_C^{\textsc{Full}}\|
\le \frac{\mu}{a_{\min}+\mu}\|\Delta_C\|.
\label{eq:deltaC-full}
\end{equation}

\paragraph{Step 2: History responds linearly to content changes.}
Lemma~\ref{lem:interp} also gives
$w_H^{\textsc{RC}}(\mu)-w_H^{\textsc{Full}} = -A_{HH}^{-1}A_{HC}(w_C^{\textsc{RC}}(\mu)-w_C^{\textsc{Full}})$.
Let $B:=A_{HH}^{-1}A_{HC}$.
Then the full parameter displacement $\delta(\mu):=w^{\textsc{RC}}(\mu)-w^{\textsc{Full}}$ satisfies
\[
\delta(\mu)=\begin{bmatrix}\delta_C(\mu)\\ \delta_H(\mu)\end{bmatrix}
=
\begin{bmatrix}I\\ -B\end{bmatrix}\delta_C(\mu),
\qquad
\delta_C(\mu)=w_C^{\textsc{RC}}(\mu)-w_C^{\textsc{Full}}.
\]
Therefore,
\begin{eqnarray}
\|\delta(\mu)\|_2^2
&=&
\|\delta_C(\mu)\|_2^2+\|B\delta_C(\mu)\|_2^2\nonumber\\
&\le&
\big(1+\|B\|_{\mathrm{op}}^2\big)\,\|\delta_C(\mu)\|_2^2
=
\kappa_H\,\|\delta_C(\mu)\|_2^2,
\label{eq:delta-full-norm}
\end{eqnarray}
where $\kappa_H=1+\|A_{HH}^{-1}A_{HC}\|_{\mathrm{op}}^2$.

\paragraph{(A) Bounded regret to \textsc{Full} on the source $\mathcal P$.}
By definition,
\begin{align}
\mathcal{D}_{\mathcal P}(w^{\textsc{RC}}(\mu),w^{\textsc{Full}})
&=
\delta(\mu)^\top \Sigma_{\mathcal P}\,\delta(\mu)\nonumber\\
&\le \lambda_{\max}(\Sigma_{\mathcal P})\,\|\delta(\mu)\|_2^2
= L_{\mathcal P}\,\|\delta(\mu)\|_2^2.\nonumber    
\end{align}

Combining with \eqref{eq:delta-full-norm} and \eqref{eq:deltaC-full} yields
\begin{equation}
\mathcal{D}_{\mathcal P}(w^{\textsc{RC}}(\mu),w^{\textsc{Full}})
\le
L_{\mathcal P}\,\kappa_H\left(\frac{\mu}{a_{\min}+\mu}\right)^2\|\Delta_C\|^2.
\label{eq:DP-bound}
\end{equation}
To translate this into an excess-risk regret statement, apply Lemma~\ref{lem:disagreement-to-risk}
with $(u,v)=(w^{\textsc{RC}}(\mu),w^{\textsc{Full}})$:
\begin{align}
\big|\mathcal{E}_{\mathcal P}(w^{\textsc{RC}}(\mu))-\mathcal{E}_{\mathcal P}(w^{\textsc{Full}})\big|
\le
&2\sqrt{\mathcal{E}_{\mathcal P}(w^{\textsc{Full}})}\,\sqrt{\mathcal{D}_{\mathcal P}(w^{\textsc{RC}}(\mu),w^{\textsc{Full}})}\nonumber\\
&+
\mathcal{D}_{\mathcal P}(w^{\textsc{RC}}(\mu),w^{\textsc{Full}}),\nonumber
\end{align}
and substituting \eqref{eq:DP-bound} yields an explicit bound controlled by $\mu$.

\paragraph{(B) Bounded regret to \textsc{Content} on the target $\mathcal Q$.}
Under the cold target \eqref{eq:coldQ}, $x_H\equiv 0$ a.s., hence only the content block contributes to
prediction disagreement and excess risk.
Let $\tilde\delta_C(\mu):=w_C^{\textsc{RC}}(\mu)-w_C^{(1)}$.
From Lemma~\ref{lem:interp} we have $\tilde\delta_C(\mu)=-(I-W_\mu)\Delta_C=-A_0(A_0+\mu I)^{-1}\Delta_C$ and
\begin{equation}
\|\tilde\delta_C(\mu)\|
\le \frac{a_{\max}}{a_{\max}+\mu}\|\Delta_C\|.
\label{eq:deltaC-content}
\end{equation}
Therefore, with $\Sigma_{\mathcal Q,CC}$ the content covariance under $\mathcal Q$,
\begin{align}
\mathcal{D}_{\mathcal Q}(w^{\textsc{RC}}(\mu),w^{\textsc{Content}})
&=
\tilde\delta_C(\mu)^\top \Sigma_{\mathcal Q,CC}\,\tilde\delta_C(\mu)\nonumber\\
&\le
\lambda_{\max}(\Sigma_{\mathcal Q,CC})\,\|\tilde\delta_C(\mu)\|^2
=
L_{\mathcal Q}\,\|\tilde\delta_C(\mu)\|^2,\nonumber    
\end{align}

and by \eqref{eq:deltaC-content} we obtain
\begin{equation}
\mathcal{D}_{\mathcal Q}(w^{\textsc{RC}}(\mu),w^{\textsc{Content}})
\le
L_{\mathcal Q}\left(\frac{a_{\max}}{a_{\max}+\mu}\right)^2\|\Delta_C\|^2.
\label{eq:DQ-bound}
\end{equation}
Applying Lemma~\ref{lem:disagreement-to-risk} with $(u,v)=(w^{\textsc{RC}}(\mu),w^{\textsc{Content}})$ yields
an excess-risk regret bound of the same form (up to the additional factor involving
$\sqrt{\mathcal{E}_{\mathcal Q}(w^{\textsc{Content}})}$).

\paragraph{(C) Improvement over \textsc{Full} on $\mathcal Q$ for large enough $\mu$.}
Assume $m_{\mathcal Q}:=\mathcal{E}_{\mathcal Q}(w^{\textsc{Full}})-\mathcal{E}_{\mathcal Q}(w^{\textsc{Content}})>0$.
A sufficient condition for $\mathcal{E}_{\mathcal Q}(w^{\textsc{RC}}(\mu))\le \mathcal{E}_{\mathcal Q}(w^{\textsc{Full}})$ is
\begin{align}
\mathcal{E}_{\mathcal Q}(w^{\textsc{RC}}(\mu))-\mathcal{E}_{\mathcal Q}(w^{\textsc{Content}})
\le m_{\mathcal Q}.\nonumber
\end{align}
By Lemma~\ref{lem:disagreement-to-risk} with $v=w^{\textsc{Content}}$, we have
\begin{align}
&\mathcal{E}_{\mathcal Q}(w^{\textsc{RC}}(\mu))-\mathcal{E}_{\mathcal Q}(w^{\textsc{Content}})\nonumber\\
&\le
2\sqrt{\mathcal{E}_{\mathcal Q}(w^{\textsc{Content}})\mathcal{D}_{\mathcal Q}(w^{\textsc{RC}}(\mu),w^{\textsc{Content}})}
+
\mathcal{D}_{\mathcal Q}(w^{\textsc{RC}}(\mu),w^{\textsc{Content}}).\nonumber
\end{align}
Let $s(\mu):=\sqrt{\mathcal{D}_{\mathcal Q}(w^{\textsc{RC}}(\mu),w^{\textsc{Content}})}$.
The right-hand side is $s(\mu)^2+2\sqrt{\mathcal{E}_{\mathcal Q}(w^{\textsc{Content}})}\,s(\mu)$, which is at most $m_{\mathcal Q}$
whenever
\begin{equation}
s(\mu)\le
\sqrt{\mathcal{E}_{\mathcal Q}(w^{\textsc{Full}})}-\sqrt{\mathcal{E}_{\mathcal Q}(w^{\textsc{Content}})}.
\label{eq:mu-large-sufficient}
\end{equation}
Using the disagreement bound \eqref{eq:DQ-bound}, condition \eqref{eq:mu-large-sufficient} holds if
\[
\sqrt{L_{\mathcal Q}}\,\frac{a_{\max}}{a_{\max}+\mu}\,\|\Delta_C\|
\;\le\;
\sqrt{\mathcal{E}_{\mathcal Q}(w^{\textsc{Full}})}-\sqrt{\mathcal{E}_{\mathcal Q}(w^{\textsc{Content}})}.
\]
Rearranging yields an explicit lower bound on $\mu$.
In the cold-start regime where $\mathcal{E}_{\mathcal Q}(w^{\textsc{Content}})$ is small,
the right-hand side is close to $\sqrt{m_{\mathcal Q}}$, recovering the simpler (but more conservative)
condition in \eqref{eq:mu-large} up to constant factors.

\paragraph{(D) Improvement over \textsc{Content} on $\mathcal P$ for small enough $\mu$.}
Assume $m_{\mathcal P}:=\mathcal{E}_{\mathcal P}(w^{\textsc{Content}})-\mathcal{E}_{\mathcal P}(w^{\textsc{Full}})>0$.
A sufficient condition for $\mathcal{E}_{\mathcal P}(w^{\textsc{RC}}(\mu))\le \mathcal{E}_{\mathcal P}(w^{\textsc{Content}})$ is
\[
\sqrt{\mathcal{E}_{\mathcal P}(w^{\textsc{RC}}(\mu))}
\le
\sqrt{\mathcal{E}_{\mathcal P}(w^{\textsc{Content}})}.
\]
By Lemma~\ref{lem:disagreement-to-risk} with $v=w^{\textsc{Full}}$,
\[
\sqrt{\mathcal{E}_{\mathcal P}(w^{\textsc{RC}}(\mu))}
\le
\sqrt{\mathcal{E}_{\mathcal P}(w^{\textsc{Full}})}+\sqrt{\mathcal{D}_{\mathcal P}(w^{\textsc{RC}}(\mu),w^{\textsc{Full}})}.
\]
Hence it suffices that
\begin{equation}
\sqrt{\mathcal{D}_{\mathcal P}(w^{\textsc{RC}}(\mu),w^{\textsc{Full}})}
\le
\sqrt{\mathcal{E}_{\mathcal P}(w^{\textsc{Content}})}-\sqrt{\mathcal{E}_{\mathcal P}(w^{\textsc{Full}})}.
\label{eq:mu-small-sufficient}
\end{equation}
Using \eqref{eq:DP-bound}, condition \eqref{eq:mu-small-sufficient} holds if
\[
\sqrt{L_{\mathcal P}\kappa_H}\,\frac{\mu}{a_{\min}+\mu}\,\|\Delta_C\|
\;\le\;
\sqrt{\mathcal{E}_{\mathcal P}(w^{\textsc{Content}})}-\sqrt{\mathcal{E}_{\mathcal P}(w^{\textsc{Full}})}.
\]
Rearranging yields an explicit upper bound on $\mu$.

This completes the proof.
\end{proof}

\subsection{Proof of Corollary~\ref{cor:mu-interval}}
\label{app:proof-cor}

\begin{proof}[Proof of Corollary~\ref{cor:mu-interval}]
The displayed displacement bounds are exactly \eqref{eq:disp-bounds1} in Lemma~\ref{lem:interp}.
For the Pareto-curve statement, combine the disagreement bounds \eqref{eq:DP-bound} and \eqref{eq:DQ-bound}:
the source disagreement increases with $\mu$ as $\big(\frac{\mu}{a_{\min}+\mu}\big)^2$,
while the target disagreement decreases as $\big(\frac{a_{\max}}{a_{\max}+\mu}\big)^2$.
The existence of a feasible $\mu$ interval that satisfies both improvement conditions follows from
parts (C) and (D) of Theorem~\ref{thm:pareto} whenever the corresponding margins are positive.
\end{proof}


\section{GBDT Risk Bound}
\label{app: gbdt-bound}

We now state a conservative guarantee that parallels the ``controlled refinement'' interpretation of Theorem~\ref{thm:pareto}.
Because boosting is nonconvex in tree structure, we do not attempt to characterize the \emph{optimal} joint model.
Instead, we show that RC \emph{inherits} the anchor's target performance up to a correction term that is explicitly controlled by Stage~2 capacity.
This already yields a verifiable sufficient condition for improving over a belief-dominant baseline on the target population.

\begin{proposition}[Base-margin RC: target-risk stability under bounded correction]
\label{prop:gbdt-stability}
Consider squared loss and a target population $\mathcal{Q}$ over merit features $x_{\merit}$.
Let $F_{\anc}(x_{\merit})$ be any fixed Stage~1 predictor and let $F_{\rc}=F_{\anc}+G$ be the Stage~2 refined predictor.
Assume that on the target population, the Stage~2 correction is uniformly bounded:
\begin{equation}
|G(x_{\merit},x_{\belief})|\le B
\qquad\text{for all $(x_{\merit},x_{\belief})$ in the support of $\mathcal{Q}$.}
\label{eq:Gbound}
\end{equation}
Then the excess target risk satisfies
\begin{equation}
\mathcal{R}_{\mathcal{Q}}(F_{\rc})
\;\le\;
\mathcal{R}_{\mathcal{Q}}(F_{\anc})
\;+\;
2B\,\sqrt{\mathcal{R}_{\mathcal{Q}}(F_{\anc})}
\;+\;
B^2,
\label{eq:gbdt-risk-bound}
\end{equation}
where $\mathcal{R}_{\mathcal{Q}}(F):=\mathbb{E}_{\mathcal{Q}}\big[(F(x)-\mathbb{E}[y\!\mid\!x_{\merit}])^2\big]$ is the (noise-free) squared error on $\mathcal{Q}$.
\end{proposition}

\paragraph{Proof.}
Let $m(x_{\merit})=\mathbb{E}[y\mid x_{\merit}]$.
Then
$\mathcal{R}_{\mathcal{Q}}(F_{\rc})=\mathbb{E}[(F_{\anc}-m+G)^2]
=\mathbb{E}[(F_{\anc}-m)^2] + 2\mathbb{E}[(F_{\anc}-m)G]+\mathbb{E}[G^2]$.
By Cauchy--Schwarz and the bounds $|G|\le B$, $\mathbb{E}[G^2]\le B^2$, and
$$|\mathbb{E}[(F_{\anc}-m)G]|\le \sqrt{\mathcal{R}_{\mathcal{Q}}(F_{\anc})}\,\sqrt{\mathbb{E}[G^2]}\le B\sqrt{\mathcal{R}_{\mathcal{Q}}(F_{\anc})},
$$
yielding Eq.~\eqref{eq:gbdt-risk-bound}.
\hfill$\square$

\paragraph{Implications and Practical Considerations}
Condition~\eqref{eq:Gbound} can be enforced or monitored in practice.
If each Stage~2 tree has leaf outputs bounded by $b$ (through leaf regularization and clipping), then $B \le \eta_2 T_2 b$.
More generally, $B$ can be estimated empirically on a target proxy (cold validation) by tracking the distribution of $|G|$.
A direct corollary is that if the Stage~1 anchor beats a baseline full model on the target by a margin $\gamma$, then choosing Stage~2 such that the right-hand side in Eq.~\eqref{eq:gbdt-risk-bound} increases risk by less than $\gamma$ ensures RC strictly improves on the target.
This parallels Theorem~\ref{thm:pareto}: RC helps when the anchor is strong on $\mathcal{Q}$ and Stage~2 is a controlled refinement.

\end{document}